\documentclass{article}

\usepackage{microtype}
\usepackage{graphicx}
\usepackage{subfigure}
\usepackage{booktabs} %

\usepackage{hyperref}       %
\usepackage{url}            %
\usepackage{booktabs}       %
\usepackage{amsfonts}       %
\usepackage{microtype}      %
\usepackage{graphicx}

\usepackage{lipsum}  

\usepackage[final]{neurips_2019}
\usepackage{hyperref}

\usepackage{qiangstyle}
\newcommand{\Dset}{\D}
\newcommand{\Sset}{\mathcal{S}}
\newcommand{\Aset}{\mathcal{A}}
\newcommand{\Rset}{\mathbb{R}}
\newcommand{\Pset}{\mathcal{P}}
\newcommand{\mt}{{\operatorname{T}}}
\newcommand{\mi}{{-1}}
\newcommand{\defeq}{:=}
\newcommand{\norm}[1]{\left\|#1\right\|}

\newcommand{\secref}[1]{Section~\ref{#1}}
\newcommand{\thmref}[1]{Theorem~\ref{#1}}
\newcommand{\corref}[1]{Corollary~\ref{#1}}

\newcommand{\Lrghat}{\hat{L}_{\operatorname{RG}}}

\newcommand{\Lfvihat}{\hat{L}_{\operatorname{FVI}}}
\newcommand{\Lk}{L_{\k}}
\newcommand{\Lkhat}{\hat{L}_{\k}}
\newcommand{\Lktilde}{\tilde{L}_{\k}}
\newcommand{\thetaTD}{\hat{\theta}_{\operatorname{TD}}}
\newcommand{\sbar}{\bar{s}}
\newcommand{\abar}{\bar{a}}
\newcommand{\wb}{\pmb{w}}
\newcommand{\pv}{\rho}

\newcommand{\Fset}{\mathcal{F}}

\newcommand{\lihong}[1]{\textcolor{blue}{[LL: #1]}}

\title{\huge A Kernel Loss for Solving the Bellman Equation}

\author{
  Yihao Feng \\
  UT Austin \\
  \texttt{yihao@cs.utexas.edu}\\
  \And   
  Lihong Li \\ 
  Google Research \\
  \texttt{lihong@google.com}\\
  \And 
  Qiang Liu \\
  UT Austin \\
  \texttt{lqiang@cs.utexas.edu}
}

\date{}

\begin{document}
\maketitle

\begin{abstract}
Value function learning plays a central role in many state-of-the-art reinforcement-learning algorithms.  Many popular algorithms like Q-learning do not optimize any objective function, but are fixed-point iterations of some variants of Bellman operator that are not necessarily a contraction.  As a result, they may easily lose convergence guarantees, as can be observed in practice.
In this paper, we propose a novel loss function, which can be optimized using standard gradient-based methods with guaranteed convergence. 
The key advantage is that its gradient can be easily approximated using sampled transitions, avoiding the need for double samples required by prior algorithms like residual gradient.  Our approach may be combined with general function classes such as neural networks, using either on- or off-policy data, and is shown to work reliably and effectively in several benchmarks, including classic problems where standard algorithms are known to diverge.
\end{abstract}

\section{Introduction}

The goal of a reinforcement learning (RL) agent is to optimize its policy to maximize the long-term return through repeated interaction with an external environment.  The interaction is often modeled as a Markov decision process, whose value functions are the unique fixed points of their corresponding \emph{Bellman operators}.   Many state-of-the-art algorithms, including TD($\lambda$), Q-learning and actor-critic, have value function learning as a key component~\citep{sutton18reinforcement,szepesvari10algorithms}.

A fundamental property of the Bellman operator is that it is a contraction in the value function space in the $\ell_\infty$-norm~\citep{puterman94markov}.  Therefore, starting from any bounded initial function, with repeated applications of the operator, the value function converges to the true value function.  A number of algorithms are directly inspired by this property, such as temporal difference~\citep{sutton88learning} and its many variants~\citep{bertsekas96neuro,sutton18reinforcement,szepesvari10algorithms}.  Unfortunately, when function approximation such as neural networks is used to represent the value function in large-scale problems, the critical property of contraction is generally lost~\citep[e.g.,][]{boyan95generalization,baird95residual,tsitsiklis97analysis}, except in rather restricted cases~\citep[e.g.,][]{gordon95stable,tsitsiklis97analysis}.  Not only is this instability one of the core theoretical challenges in RL, but it also has broad practical significance, given the growing popularity of algorithms like DQN~\citep{mnih15human}, A3C~\citep{mnih16asynchronous} and their many variants~\citep[e.g.,][]{gu16continuous,schulman16high,wang16dueling,wu17scalable}, whose stability largely depends on the contraction property.  The instability becomes even harder to avoid, when training data (transitions) are sampled from an off-policy distribution, a situation known as the \emph{deadly triad}~\citep[Sec.~11.3]{sutton18reinforcement}.

The brittleness of Bellman operator's contraction property has inspired a number of works that aim to reformulate value function learning as an optimization problem, where standard algorithms like stochastic gradient descent can be used to minimize the objective, without the risk of divergence (under mild and typical assumptions).  One of the earliest attempts is residual gradient, or RG~\citep{baird95residual}, which relies on minimizing squared temporal differences.  The algorithm is convergent, but its objective is not necessarily a good proxy due to a well-known ``double sample'' problem.  As a result, it may converge to an inferior solution; see Sections~\ref{sec:background} and \ref{sec:experiments} for further details and numerical examples.  This drawback is inherited by similar algorithms like PCL~\citep{nachum17bridging,nachum2018trustpcl}.

Another line of work seeks alternative objective functions, the minimization of which leads to desired value functions~\citep[e.g.,][]{sutton09fast,maei11gradient,liu15finite,dai2017learning}.  Most existing works are either for linear approximation, or for evaluation of a fixed policy.  An exception is the SBEED algorithm~\citep{dai18sbeed}, which transforms the Bellman equation to an equivalent saddle-point problem, and can use nonlinear function approximations.  While SBEED is provably convergent under fairly standard conditions, it relies on solving a minimax problem, whose optimization can be rather challenging in practice, especially with nonconvex approximation classes like neural networks.

In this paper, we propose a novel loss function for value function learning.  It avoids the double-sample problem (unlike RG), and can be easily estimated and optimized using sampled transitions (in both on- and off-policy scenarios).  This is made possible by leveraging an important property of integrally strictly positive definite kernels~\citep{stewart1976positive,sriperumbudur2010hilbert}.  This new objective function allows us to derive simple yet effective algorithms to approximate the value function, without risking instability or divergence (unlike TD algorithms),
or solving a more sophisticated saddle-point problem (unlike SBEED).   
Our approach also allows great flexibility in choosing the value function approximation classes, including nonlinear ones like neural networks.  
Experiments in several benchmarks demonstrate the effectiveness of our method, for both policy evaluation and optimization problems.  We will focus on the batch setting (or the growing-batch setting with a growing replay buffer), and leave the online setting for future work.

\section{Background}
\label{sec:background}

This section starts with necessary notation and background information, then reviews two representative algorithms that work with general, nonlinear (differentiable) function classes.

\paragraph{Notation.} A Markov decision process (MDP) is denoted by $M = \langle \Sset, \Aset, P, R, \gamma)$, where $\Sset$ is a (possibly infinite) state space, $\Aset$ an action space, $P(s'~|~s, a)$ the transition probability, 
$R(s, a)$ the average immediate reward, 
and $\gamma\in (0,1)$ a discount factor.
The value function of a policy $\pi:\Sset\mapsto\Rset_+^\Aset$, denoted $V^\pi(s) \defeq 
\E\left [\sum_{t=0}^\infty \gamma^t  R(s_t,a_t)~|~ s_0 = s,a_t\sim\pi(\cdot,s_t)\right]\,$, measures the expected long-term return of a state.  It is well-known that $V=V^\pi$ is the unique 
solution to the \emph{Bellman equation}~\citep{puterman94markov}, $V = \B_\pi V$, where $\B_\pi:\Rset^\Sset\to\Rset^\Sset$ is the \emph{Bellman operator}, defined by 
$$
\B_\pi V(s) \defeq \E_{a \sim \pi(\cdot|s), s'\sim P(\cdot|s,a)}[R(s, a) + \gamma V(s') ~|~s]\,.
$$
While we develop and analyze our approach mostly for $\B_\pi$ given a fixed $\pi$ (policy evaluation), we will also extend the approach to the controlled case of policy optimization, where the corresponding Bellman operator becomes
\[
\B V(s) \defeq \max_{a} \E_{s'\sim P(\cdot|s,a)}[R(s,a) + \gamma V(s') ~|~s,a]\,.
\]
The unique fixed point of $\B$ is known as the optimal value function, denoted $V^*$; that is, $\B V^* = V^*$.

Our work is built on top of an alternative to the fixed-point view above:
given some fixed distribution $\mu$ whose support is $\Sset$, %
$V^\pi$ is the unique minimizer of the \emph{squared Bellman error}:
$$
L_2(V) \defeq \norm{\B_\pi V - V}_\mu^2 = \E_{s\sim\mu} \big[ \left (
\B_\pi V(s) - V(s)
\right )^2 \big]\,.
$$
Denote by $\R_\pi V \defeq \B_\pi V - V$ the Bellman error operator. 
With a set $\Dset = \{(s_i, a_i, r_i, s_i')\}_{1\le i \le n}$ of transitions where $a_i \sim \pi(\cdot|s_i)$, 
the Bellman operator in state $s_i$ can be approximated by \emph{bootstrapping:} 
$\hat \B_\pi V(s_i) \defeq 
r_i + \gamma V(s_i')$.  Similarly, $\hat \R_\pi V(s_i) \defeq 
r_i + \gamma V(s_i') - V(s_i)$.  
Clearly, one has %
$\E[\hat \B_\pi V_\theta(s_i)|s_i] = \B_\pi V_\theta(s_i)$ and  $\E[\hat \R_\pi V_\theta(s_i)|s_i] = \R_\pi V_\theta(s_i)$. 
In the literature, $\hat \B_\pi V_\theta(s_i) - V_\theta(s_i)$ is also known as the temporal difference or TD error, whose expectation is the Bellman error.
For notation, all distributions are equivalized to their probability density functions in the work. 

\paragraph{Basic Algorithms.}
We are interested in estimating $V^\pi$,
from a parametric family $\{V_\theta \colon \theta\in \Theta\}$, from data $\Dset$. 
The \emph{residual gradient} algorithm~\citep{baird95residual} minimizes the \emph{squared TD error}:
\begin{equation}
\Lrghat(V_\theta) \defeq \frac{1}{n} \sum_{i=1}^n \left(\hat\B_\pi V_\theta(s_i) - V_\theta(s_i)\right)^2, \label{equ:rg}
\end{equation}
with gradient descent update $\theta_{t+1} = \theta_t  - \epsilon \nabla_\theta\Lrghat(V_{\theta_t})$, where 
$$
\nabla_\theta\Lrghat(V_\theta) = \frac{2}{n} \sum_{i=1}^n \left(\big(\hat\B_\pi V_\theta(s_i) - V_\theta(s_i)\big) \cdot \nabla_\theta\big(\hat\B_\pi V_\theta(s_i) - V_\theta(s_i)\big)\right).
$$
However, the objective in \eqref{equ:rg} is a biased and inconsistent estimate of the squared Bellman error. This is because  $\E_{s\sim\mu}[\Lrghat(V)] = L_2(V) + \E_{s\sim\mu}\big[\var(\hat\B_\pi V(s)|s)\big] \neq L_2(V)$, where  
there is an extra term that involves the conditional variance of the empirical Bellman operator, which does not vanish unless the state transitions are deterministic. %
As a result, RG can converge to incorrect value functions (see also \secref{sec:experiments}).
With random transitions, correcting the bias 
requires double samples (i.e., at least two independent samples of $(r,s')$ for the same $(s,a)$ pair) 
to estimate the conditional variance. %

More popular algorithms in the literature are instead based on fixed-point iterations, using $\hat\B_\pi$ to construct a target value to update $V_\theta(s_i)$.  An example is \emph{fitted value iteration}, or FVI~\citep{bertsekas96neuro,munos08finite}, which includes as special cases the empirically successful DQN and variants, and also serves as a key component in many state-of-the-art actor-critic algorithms. 
In its basic form, FVI starts from an initial $\theta_0$, and iteratively updates the parameter by
\begin{align} \label{equ:fvi}
 \theta_{t+1} = \argmin_{\theta\in\Theta}\left\{ \Lfvihat^{(t+1)}(V_\theta) \defeq \frac{1}{n}\sum_{i=1}^n \left ( V_\theta(s_i) - \hat\B_\pi V_{\theta_t}(s_i)\right )^2\right\}\,.
\end{align}
Different from RG, when gradient-based methods are applied to solve \eqref{equ:fvi}, 
the current parameter $\theta_t$ is treated as a constant: 
$\nabla_\theta\Lfvihat^{(t+1)}(V_\theta)
= \frac{2}{n} \!\sum_{i=1}^n \big(V_\theta(s_i) - \hat\B_\pi V_{\theta_t}(s_i)\big) \nabla_\theta V_\theta(s_i)$.
TD(0)~\citep{sutton88learning} may be viewed as a stochastic version of FVI, where a single sample (i.e., $n=1$) is drawn randomly (either from a stream of transitions or from a replay buffer) to estimate the gradient of \eqref{equ:fvi}.

Being fixed-point iteration methods, 
FVI-style algorithms \emph{do not} optimize any objective function, and their convergence is guaranteed only in rather restricted cases~\citep[e.g.,][]{gordon95stable,tsitsiklis97analysis,antos08learning}.  Such divergent behavior is well-known and empirically observed~\citep{baird95residual,boyan95generalization}; see \secref{sec:experiments} for more numerical examples. 
It creates substantial difficulty in parameter tuning and model selection in practice.

\section{Kernel Loss for Policy Evaluation}
\label{sec:theory}

Much of the algorithmic challenge described earlier lies in the difficulty in estimating squared Bellman error from data.  In this section, we address this difficulty by proposing a new 
loss function that is more amenable to statistical estimation from  empirical data.  Proofs are deferred to the appendix.

Our framework relies on  
an \emph{integrally strictly positive definite (ISPD) kernel} %
$\k\colon\Sset\times\Sset\to\RR$, which is a symmetric bi-variate function that satisfies
$\norm{f}_{\k}^2 \defeq \int_{\Sset^2} \k(s,\bar s)f(s) f(\bar s) ~ds~d\bar s > 0$,
for any non-zero $L_2$-integrable function $f$. 
For simplicity, we consider two functions $f$ and $g$ equal if $(f-g)$ has a zero $L_2$ norm. 
We call $\norm{f}_{\k}$ the $\k$-norm of $f$.  Many commonly used kernels, such as Gaussian RBF kernel $\k(s,\bar s) = \exp(-\norm{s-\bar s}_2^2/h)$ is ISPD. 
More discussion on ISPD kernels can be found in \citet{stewart1976positive} and \citet{sriperumbudur2010hilbert}.

\subsection{The New Loss Function}

Recall that $\R_\pi V = \B_\pi V - V$ is the Bellman error operator. 
Our new loss function is defined by  
\begin{equation}\label{equ:ourloss}
\Lk(V) 
 = \norm{\R_\pi V}_{\k, \mu}^2  
 \defeq \E_{s,\bar s\sim \mu} \left [\k(s,\bar s)  \cdot \R_\pi V(s) \cdot \R_\pi V(\bar s) \right ]\,,
\end{equation} 
where $\mu$ is any positive density function on states $s$, and $s,\bar s\sim \mu$ means $s$ and $\bar s$ are drawn i.i.d. from $\mu$. 
Here, $\norm{\cdot}_{\k, \mu}$ is regarded as the $\k$-norm under measure $\mu$.  It is easy to show that  $\norm{f}_{\k,\mu} = \norm{f\mu}_\k$. 
Note that $\mu$ can be either 
the visitation distribution %
under policy $\pi$ (the \emph{on-policy} case),
or some other distribution  
(the \emph{off-policy} case).  Our approach handles both cases in a unified way.
The following theorem shows that the loss $\Lk$ is consistent:
\begin{thm}\label{thm:validate}
Let $\k$ be an ISPD kernel 
and assume $\mu(s) >0$,$\forall s\in \mathcal \Sset$. 
Then, $\Lk(V) \geq 0$ for any $V$; and $\Lk(V) = 0$ if and only if $V=V^\pi$. In other words, $V^\pi = \argmin_{V} \Lk(V)$.
\end{thm}

The next result relates the kernel loss to a ``dual'' kernel norm of the value function error, $V-V^\pi$.  
\begin{thm}\label{thm:Vfunction}
Under the same assumptions as \thmref{thm:validate}, we have $\Lk(V) =  \norm{V - V^\pi}_{\k^*,\mu}^2$, where $\norm{\cdot}_{\k^*,\mu}$ is the $\k^*$-norm under measure $\mu$ with a ``dual'' kernel $\k^*(s,\bar s)$, defined by 
$$
\k^*(s', \bar s') \defeq 
\E_{s,\sbar~\sim~ d^*_{\pi,
\mu}}\Big [\k(s', \bar s') + \gamma^2 \k(s, \bar s) - \gamma \big( \k(s', \bar s) + \k(s, \bar s') \big) ~|~ s', \bar s' \Big]\,,
$$
and the expectation notation is shorthand for
$
\E_{s\sim d^*_{\pi,
\mu}}[f(s)|s'] = \int f(s)  d^*_{\pi,\mu}(s|s')  ds\,,
$
with
$$d^*_{\pi,\mu}(s|s') \defeq {\sum_a \pi(a|s) P(s'|s,a)\mu(s)}/{\mu(s')}\,.
$$
\end{thm}
  The norm involves a quantity, $d^*_{\pi,\mu}(s|s')$, which may be heuristically viewed as a ``backward'' conditional probability of state $s$ conditioning on observing the next state $s'$ (but note that $ d^*_{\pi,\mu}(s|s')$ is not normalized to sum to one unless $\mu=d_\pi$).  %

\paragraph{Empirical Estimation}
The key advantage of the new loss $\Lk$ is that it can be easily estimated and optimized from observed transitions, without requiring double samples.  %
Given a set of empirical data $\D = \{(s_i,a_i,r_i,s_i')\}_{1\le i \le n}$, %
a way to estimate $\Lk$ is to use the so-called \emph{V-statistics}, 
\begin{equation} \label{equ:vstat} 
\Lkhat(V_\theta) 
\defeq \frac{1}{n^2}
\sum_{1 \le i, j \le n}
\k(s_i, s_j) \cdot \hat\R_\pi V_\theta(s_i) \cdot \hat\R_\pi V_\theta (s_j)\,.
\end{equation}
Similarly, the gradient 
$\nabla_\theta \Lk(V_\theta)  = 2\E_{\mu}[\k(s,\bar s) 
\R_\pi V_\theta(s)  \nabla_\theta (\R_\pi V_\theta(\bar s))]$ can be estimated by %
$$
\nabla_\theta \Lkhat(V_\theta) \defeq \frac{2}{n^2}
\sum_{1 \le i, j \le n}  
\k(s_i, s_j) \cdot \hat\R_\pi V_\theta(s_i) \cdot\nabla_\theta \hat\R_\pi V_\theta (s_j)\,. 
$$
Note that while calculating the exact gradient requires $O(n^2)$ computation, in practice we may use stochastic gradient descent on mini-batches of data instead.
The precise formulas for unbiased estimates of the gradient of the kernel loss using a subset of samples are given in Appendix~\ref{sec:batch_lv}.

\paragraph{Remark} \hspace{-3mm} (unbiasedness) \hspace{1mm}
An alternative approach is to use the \emph{U-statistics}, 
which removes the diagonal ($i=j$) terms in the pairwise average in \eqref{equ:vstat}. 
In the case of i.i.d. samples, it is known that U-statistics forms an unbiased estimate of the true gradient, but may have higher variance than the V-statistics. 
In our experiments, we observe that V-statistics works better than U-statistics.

\paragraph{Remark} \hspace{-3mm} (consistency) \hspace{1mm}
Following standard statistical approximation theory \citep[e.g.,][]{serfling2009approximation}, 
both U/V-statistics provide \emph{consistent} estimation of the expected quadratic quantity given 
the sample is weakly dependent and satisfies certain mixing condition \citep[e.g.,][]{denker1983u, beutner2012deriving};  %
this often amounts to saying that $\{s_i\}$ forms a Markov chain that converges to its stationary distribution $\mu$ sufficiently fast. 
This is in contrast to the gradient computed by residual gradient, which is known to be inconsistent in general.

\paragraph{Remark} 
Another advantage of our kernel loss is that we have $\Lk(V) = 0$ iff $V = V^\pi$. Therefore, the magnitude of the empirical loss $\Lkhat(V)$ reflects the closeness of $V$ to the true value function $V^\pi$. 
In fact, by using methods from kernel-based hypothesis testing~\citep[e.g.,][]{gretton2012kernel,liu2016kernelized,chwialkowski2016kernel}, one can design statistically calibrated methods to test if $V = V^\pi$ has been achieved, which may be useful for designing efficient exploration strategies.  In this work, we focus on estimating $V^\pi$ and leave it as future work to test value function proximity.

\subsection{Interpretations of the Kernel Loss}

We now provide some insights into the new loss function, based on two interpretations.

\paragraph{Eigenfunction Interpretation}
Mercer's theorem implies the following decomposition
 \begin{align}\label{equ:mercer}
     \k(s, \bar s) = \sum_{i =1 }^\infty \lambda_i e_i(s) e_i(\bar s)\,, 
 \end{align}
of any continuous positive definite kernel on a compact domain, where $\{e_i\}_{i =1}^\infty$ is a countable set of orthonormal eigenfunctions w.r.t. $\mu$ (i.e., $\E_{s\sim \mu}[e_i(s) e_j(s)] = \mathbf{1}\{i=j\}$),
and $\{\lambda_i\}_{i=1}^\infty$ are their eigenvalues.  
For ISPD kernels, all the eigenvalues must be positive: $\forall i,~\lambda_i > 0$.

The following shows that $\Lk$ is a squared projected Bellman error in the space 
spanned by $\{e_i\}_{i=1}^\infty$.
\begin{pro}\label{thm:eigen}
If \eqref{equ:mercer} holds, then
$$
\Lk(V) = \sum_{i=1}^\infty 
\lambda_i  \left ( \E_{s\sim \mu} \left [ \R_\pi V(s) \times  e_i(s) \right ] \right)^2\,.
$$%
Moreover, if $\{e_i\}$ is a complete orthonormal basis of $L_2$-space under measure $\mu$, then the $L_2$ loss is
$$
L_2(V) = \sum_{i=1}^\infty   \left ( \E_{s\sim \mu} \left [ \R_\pi V(s) \times  e_i(s) \right ] \right)^2\,. %
$$
Therefore, $\Lk(V)\leq \lambda_{\max} L_2(V)$, where $\lambda_{\max} \defeq \max_i \{\lambda_i\}$. 
\end{pro}

This result shows that the eigenvalue $\lambda_i$ controls the contribution of the projected Bellman error to the eigenfunction $e_i$ in $\Lk$.  It may be tempting to have $\lambda_i \equiv 1$, in which $\Lk(V)=L_2(V)$, 
but the Mercer expansion in \eqref{equ:mercer} can diverge to infinity, resulting in an ill-defined kernel $K(s,\bar s)$.  %
To avoid this, the eigenvalues must decay to zero fast enough such that $\sum_{i=1}^\infty\lambda_i <\infty$. 
Therefore, the kernel loss $\Lk(V)$ can be viewed as prioritizing the 
 projections to the eigenfunctions with larger eigenvalues. %
In typical kernels such as Gaussian RBF kernels, these dominant eigenfunctions 
are Fourier bases with low frequencies (and hence high smoothness), which may intuitively be more relevant than the  higher frequency bases for practical purposes. 

\paragraph{RKHS Interpretation} The squared Bellman error has the following variational form:
\begin{equation}
L_2(V)
= \max_{f} 
\left \{ \left (\E_{s\sim \mu} \left [ \R_\pi V(s) \times f(s) \right ] \right )^2 
\colon ~~~~~~
\E_{s\sim \mu }[({f(s)})^2] \leq 1 
\right \}
\,,
\label{equ:L2constr}
\end{equation} 
which involves finding a function 
$f$ in the unit $L_2$-ball, whose inner product with $\R_\pi V(s)$ is maximal.
Our kernel loss has a similar interpretation, with a different unit ball.

Any positive kernel $K(s,\bar s)$ is associated with a Reproducing Kernel Hilbert Space (RKHS) $\H_K$, which is the Hilbert space consisting of (the closure of) the linear span of $K(\cdot, s)$, for $s\in \mathcal S$, and satisfies the reproducing property, $f(x) = \langle f, ~ K(\cdot, x)\rangle_{\H_K}$, for any $f\in \H_K$.
RKHS has been widely used as a powerful tool in various machine learning and statistical problems; see \citet{berlinet2011reproducing, muandet2017kernel} for overviews. 

\begin{pro}\label{thm:rkhs}
Let $\H_\k$ be the RKHS of kernel $\k(s,\bar s)$, we have
\begin{equation}
\Lk(V) = \max_{f \in \H_\k} 
 \left \{ \left (\E_{s\sim \mu} \left [ \R_\pi V(s) \times f(s) \right ] \right )^2
 \colon ~~ \norm{f}_{\H_\k} \le 1
 \right \}.
\end{equation}
\end{pro}

Since RKHS is a subset of the $L_2$ space that includes smooth functions, we can again see that $\Lk(V)$ emphasizes more the projections to smooth basis functions, matching the intuitive from Theorem~\ref{thm:eigen}. 
It also draws a connection to the recent primal-dual reformulations of the Bellman equation~\citep{dai2017learning,dai18sbeed}, which formulate 
$V^\pi$ as a saddle-point of the following minimax problem:
\begin{align} \label{equ:sbeed}
\min_{V} \max_{f} %
~~ \E_{s\sim \mu} \left [ 2 \R_\pi V(s) \times f(s) -  f(s)^2 \right ].
\end{align}
This is equivalent to minimizing $L_2(V)$ as \eqref{equ:L2constr}, except that the L2 constraint is replaced by a quadratic penalty term. 
When only samples are available, the expectation in 
\eqref{equ:sbeed} is replaced by the empirical version. If the optimization domain of $f$ is \emph{unconstrained}, solving the empirical \eqref{equ:sbeed} reduces to the empirical L2 loss \eqref{equ:rg}, which yields inconsistent estimation.
Therefore, 
existing works propose to  
further constrain the optimization of $f$ in \eqref{equ:sbeed} to 
either RKHS~\citep{dai2017learning} or neural networks~\citep{dai18sbeed}, and hence derive a minimax strategy for learning $V$. 
Unfortunately, this is substantially more expensive than our method due to the cost of updating another neural network $f$ jointly;  
the minimax procedure may also make the training less stable and more difficult to converge in practice.

\subsection{Connection to Temporal Difference (TD) Methods}
\label{sec:special_cases}

We now instantiate our algorithm in the \emph{tabular} and \emph{linear} cases to gain further insights.  Interestingly, we show that our loss coincides with previous work, and as a result leads to the same value function as several classic algorithms.  Hence, the approach developed here may be considered as their strict extensions to the much more general nonlinear function approximation classes.

Again, let $\Dset$ be a set of $n$ transitions sampled from distribution $\mu$,
and linear approximation be used:  %
$V_\theta(s) = \theta^\mt \phi(s)$, 
where $\phi:S \to \Rset^d$ is a feature function, and $\theta \in \Rset^d$ is the parameter to be learned. 
The TD solution, $\thetaTD$, for either on- and off-policy cases, can be found by various algorithms~\citep[e.g.,][]{sutton88learning,boyan99least,sutton09fast,dann14policy}, and its theoretical properties have been extensively studied~\citep[e.g.,][]{tsitsiklis97analysis,lazaric12finite}.

\begin{cor} \label{cor:same_as_td}
When using a linear kernel of form 
$k(s,\bar{s})=\phi(s)^\mt\phi(\bar{s})$,  minimizing the kernel objective~\eqref{equ:vstat} gives the TD solution $\thetaTD$.
\end{cor}

\begin{remark}
The result follows from the observation that our loss becomes the Norm of the Expected TD Update (NEU) in the linear case~\citep{dann14policy}, whose minimizer coincides with $\thetaTD$.  Moreover, in finite-state MDPs, the corollary includes tabular TD as a special case, by using a one-hot vector (indicator basis) to represent states.
In this case, the TD solution coincides with that of a model-based approach~\citep{Parr08Analysis} known as \emph{certainty equivalence}~\citep{kumar86stochastic}.
It follows that our algorithm includes certainty equivalence as a special case in finite-state problems.
\end{remark}

\section{Kernel Loss for Policy Optimization}

There are different ways to extend our approach to policy optimization.  One is to use the kernel loss \eqref{equ:ourloss} inside an existing algorithm, as an alternative to RG or TD to learn $V^\pi(s)$.  For example, our loss fits naturally into an actor-critic algorithm, 
where we replace the critic update (often implemented by TD$(\lambda)$ or its variant) with our method, and the actor updating part remains unchanged. 
Another, more general way is to design a kernelized loss for $V(s)$ and policy $\pi(a|s)$ jointly, so that the policy optimization can be solved using a single optimization procedure.  
Here, we take the first approach, leveraging our method to improve the critic update step in Trust-PCL~\citep{nachum2018trustpcl}. 

Trust-PCL is based on a temporal/path consistency condition resulting from policy smoothing~\citep{nachum17bridging}.
We start with the smoothed Bellman operator,  defined by
\begin{align} \label{equ:bvsopt}
\B_\lambda V(s) = \max_{\pi(\cdot|s)\in\Pset_\Aset} \E_\pi [ R(s,a) + \gamma V(s') + \lambda H(\pi~|~s)~|~ s]\,,
\end{align}
where $\Pset_\Aset$ is the set of distributions over action space $\Aset$; the conditional expectation $\E_\pi[\cdot|s]$ denotes $a\sim \pi(\cdot|s)$, and $\lambda>0$ is a smoothing parameter; $H$ is a state-dependent entropy term: $H(\pi~|~s) \defeq -\sum_{a\in\Aset} \pi(a|s)  \log\pi(a|s)$.  Intuitively, $\B_\lambda$ is a smoothed approximation of $\B$.
It is known that $\B_\lambda$ is a $\gamma$-contraction~\citep{fox15taming}, so has a unique fixed point $V^*_\lambda$.  Furthermore, with $\lambda=0$ we recover the standard Bellman operator, and $\lambda$ smoothly controls $\|V^*_\lambda-V^*\|_\infty$~\citep{dai18sbeed}.

The entropy regularization above implies the following path consistency condition.  Let $\pi^*_\lambda$ be an optimal policy in \eqref{equ:bvsopt} for $\B_\lambda$, which yields $V^*_\lambda$.  Then, $(V,\pi)=(V^*_\lambda,\pi^*_\lambda)$ uniquely solves
\[
\forall (s,a)\in\Sset\times\Aset: \quad V(s) = R(s,a) + \gamma \E_{s'|s,a}[V(s')] - \lambda\log\pi(a|s)\,.
\]

This property inspires a natural extension of the kernel loss \eqref{equ:ourloss} to the controlled case:
\[
\Lk(V) = \E_{s,\sbar\sim\mu,a\sim\pi(\cdot|s),\bar{a}\sim\pi(\cdot|\bar{s})} [K([s, a], [\bar{s}, \bar{a}]) \cdot \R_{\pi,\lambda} V(s,a) \cdot \R_{\pi,\lambda} V(\sbar,\abar)]\,,
\]
where $\R_{\pi,\lambda} V(s,a)$ is given by
\[
\R_{\pi,\lambda} V(s,a) = R(s,a) + \gamma \E_{s'|s,a} [V(s')] - \lambda\log\pi(a|s) - V(s) \,.
\]
Given a set of transitions $\Dset=\{(s_i,a_i,r_i,s_i')\}_{1\le i\le n}$, the objective can be estimated by
\[
\Lkhat(V_\theta) = \frac{1}{n^2} \sum_{1\le i,j \le n} [\k([s_i, a_i],[s_j, a_j]) \hat{\R}_i  \hat{\R}_j]\,, 
\]
with
$$
\hat{\R}_i = r_i + \gamma V_\theta(s_i') - \lambda \log\pi_\theta(a_i|s_i) - V_\theta(s_i)\,.
$$

The U-statistics version and multi-step bootstraps can be similarly obtained~\citep{nachum17bridging}.

\section{Related Work}

In this work, we studied value function learning, one of the most-studied and fundamental problems in reinforcement learning.  The dominant approach is based on fixed-point iterations~\citep{bertsekas96neuro,szepesvari10algorithms,sutton18reinforcement}, which can risk instability and even divergence when function approximation is used, as discussed in the introduction.

Our approach exemplifies more recent efforts that aim to improve stability of value function learning by reformulating it as an optimization problem.  Our key innovation is the use of a kernel method to estimate the squared Bellman error, which is otherwise hard to estimate directly from samples, thus avoids the double-sample issue unaddressed by prior algorithms like residual gradient~\citep{baird95residual} and PCL~\citep{nachum17bridging,nachum2018trustpcl}.  As a result, our algorithm is \emph{consistent}: it finds the true value function with enough data, using sufficiently expressive function approximation classes.  Furthermore, the solution found by our algorithm minimizes the projected Bellman error, as in prior works when specialized to the same settings~\citep{sutton09fast,maei10toward,liu15finite,macua15distributed}.  However, our algorithm is more general: it allows us to use nonlinear value function classes and can be naturally implemented for policy optimization.  Compared to nonlinear GTD2/TDC~\citep{maei09convergent}, our method is simpler (without having to do a local linear expansion) and empirically more effective (as demonstrated in the next section).

As discussed in \secref{sec:theory}, our approach is related to the recently proposed SBEED algorithm~\citep{dai18sbeed} which shares many advantages with this work.  However, SBEED requires solving a minimax problem 
that can be rather challenging in practice.  In contrast, our algorithm only needs to solve a minimization problem, for which a wide range of powerful methods exist~\citep[e.g.,][]{bertsekas16nonlinear}.  Note that there exist other saddle-point formulations for RL, which is derived from the linear program of MDPs~\citep{wang2017primal,chen18scalable,dai18boosting}.   
The connection and comparison between these formulations would be interesting directions to investigate.

Our work is also related to a line of interesting work on Bellman residual minimization (BRM) based on nested optimization~\citep{antos08learning,farahmand08regularized,farahmand16regularized,hoffman11regularized,chen19information}.  They formulate the value function as the solution to a coupled optimization problem, where both the inner and outer optimization are over the same function space.  While their inner optimization plays a similar role as our use of RKHS in the kernel loss definition, our loss is derived from a different way, and decouples the representations used in inner and outer optimizations.  Furthermore, the nested optimization formulation also involves solving a minimax problem (similar to SBEED), while our approach is much simpler as it only requires solving a minimization problem.

Finally, the kernel method
has been widely used in machine learning~\citep[e.g.,][]{scholkopf2001learning, muandet2017kernel}.  In RL, authors have used kernels  
either to smooth the estimates of transition probabilities and rewards~\citep{ormoneit02kernel}, or to represent the value  function~\citep[e.g.,][]{xu05kernel,xu07kernel,taylor09kernelized}.  
Our method differs from these works in that we leverage kernels
for \emph{designing proper loss functions} to address the double-sampling problem,
while putting no constraints on %
which approximation classes to represent the value function. %
Our approach is thus expected to be more flexible and scalable in practice, allowing the value function to lie in flexible function classes like neural networks.

\newcommand{\KBO}{{K-loss~}}

\section{Experiments} 
\label{sec:experiments}

We compare our method (labelled ``K-loss'' in all experiments) with several representative baselines in both classic examples and popular benchmark problems, for both policy evaluation and optimization.

\subsection{Modified Example of \texorpdfstring{\citeauthor{tsitsiklis97analysis}}{TV}}

\figref{fig:toy_diverge}~(a) shows a modified problem of the classic example 
by \citet{tsitsiklis97analysis}, by making transitions stochastic.\footnote{Recall that the double-sample issue exists only in stochastic problems, so the modification is necessary to make the comparison to residual gradient meaningful.}  It consists of $5$ states, including $4$ nonterminal (circles) and $1$ terminal states (square), and $1$ action.  The arrows represent transitions between states.
The value function estimate is linear in the weight $\vv w = [w_1, w_2, w_3]$: for example, the leftmost and bottom-right states' values are $w_1$ and $2w_3$, respectively.  Furthermore, we set $\gamma=1$, so $V(s)$ is exact with the optimal weight $\vv w^*=[0.8, 1.0, 0]$.
In the experiment, we randomly collect $2\,000$ transition tuples for training.
We use a linear kernel in our method, so that it will find the TD solution (\corref{cor:same_as_td}).  

\figref{fig:toy_diverge}~(b\&c) show the learning curves of mean squared error ($\|V-V^*\|^2$) and weight error ($\|\pmb{w}-\vv w^*\|$) of different algorithms over iterations.
Results are consistent with theory: our method converges to the true weight $\vv w^*$, while both FVI and TD(0) diverge, and RG converges to a wrong solution.

\newcommand{\tmodfdfdf}{.28}
\begin{figure}[t]
    \centering
    \begin{tabular}{ccc} 
    \raisebox{1em}{\includegraphics[width=\tmodfdfdf\linewidth]{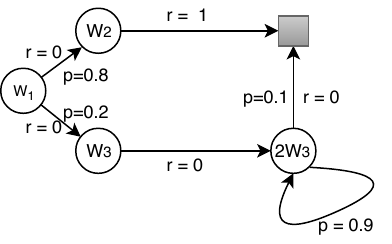}} &  
      \includegraphics[width=\tmodfdfdf\linewidth]{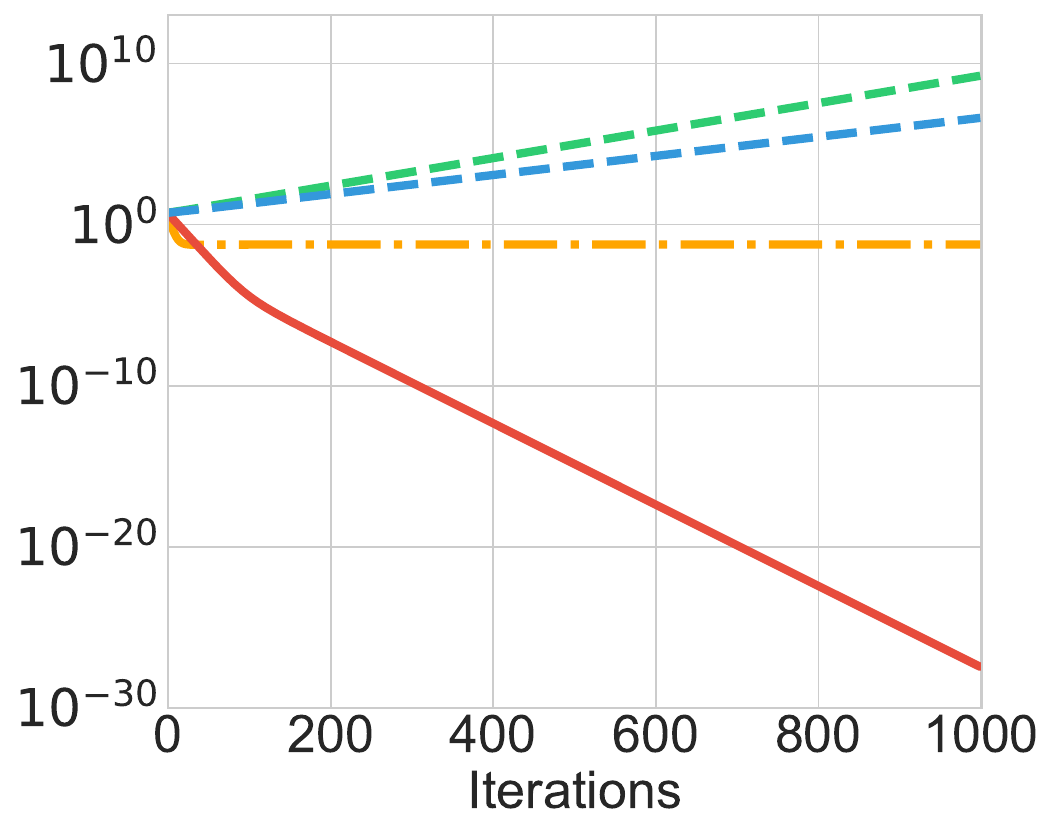} & 
      \includegraphics[width=.25\linewidth]{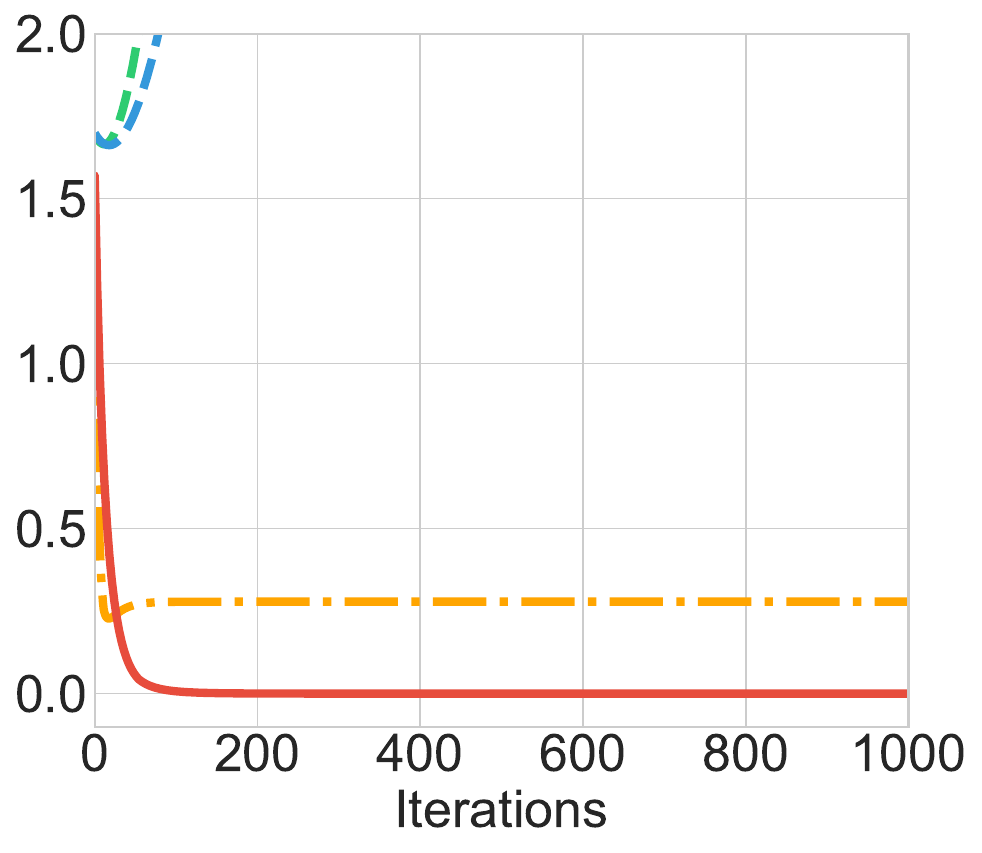}  \hspace{-1.3em}\llap{\makebox[\wd1][r]{\raisebox{4.7em}{\includegraphics[height=3.5em]{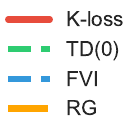}}}}
      \\  
      \small{(a) Our MDP Example} & \small{(b) MSE vs. Iteration} & (c) \small{$||\vv w - \vv w^*||$ vs. Iteration}
      \\
    \end{tabular}
    \caption{ Modified example of \cite{tsitsiklis97analysis}.}
    \label{fig:toy_diverge}
\end{figure}

\newcommand{\tmoddd}{.25}
\newcommand{\gapline}{-1.2em}
\begin{figure}[t]
    \centering
    \begin{tabular}{cccc}
        \hspace{\gapline}
        \includegraphics[width=\tmoddd\linewidth]{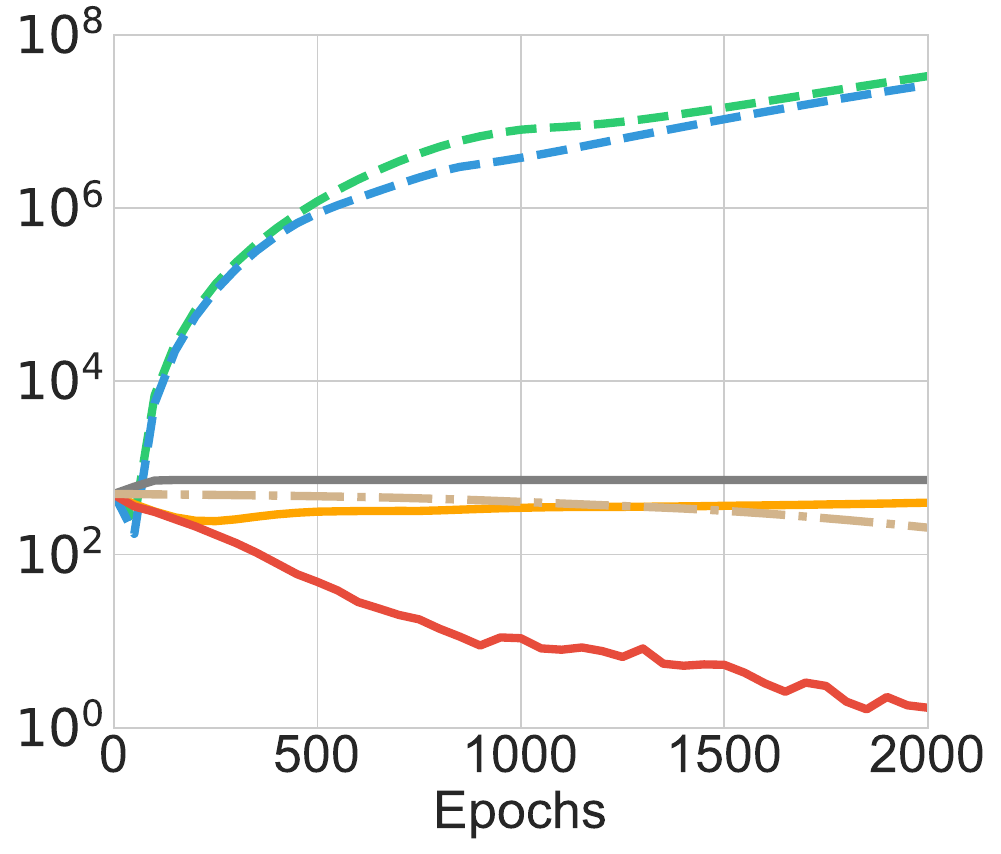}&
        \hspace{\gapline}
        \includegraphics[width=\tmoddd\linewidth]{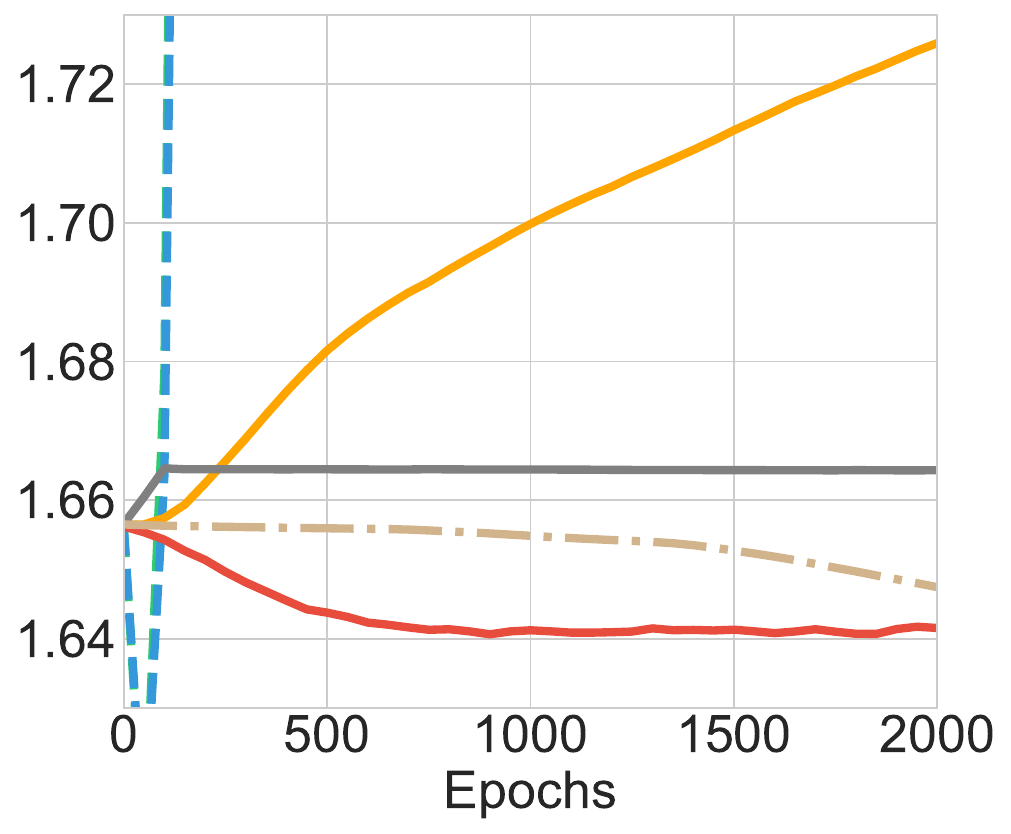}& 
        \hspace{\gapline}
        \includegraphics[width=\tmoddd\linewidth]{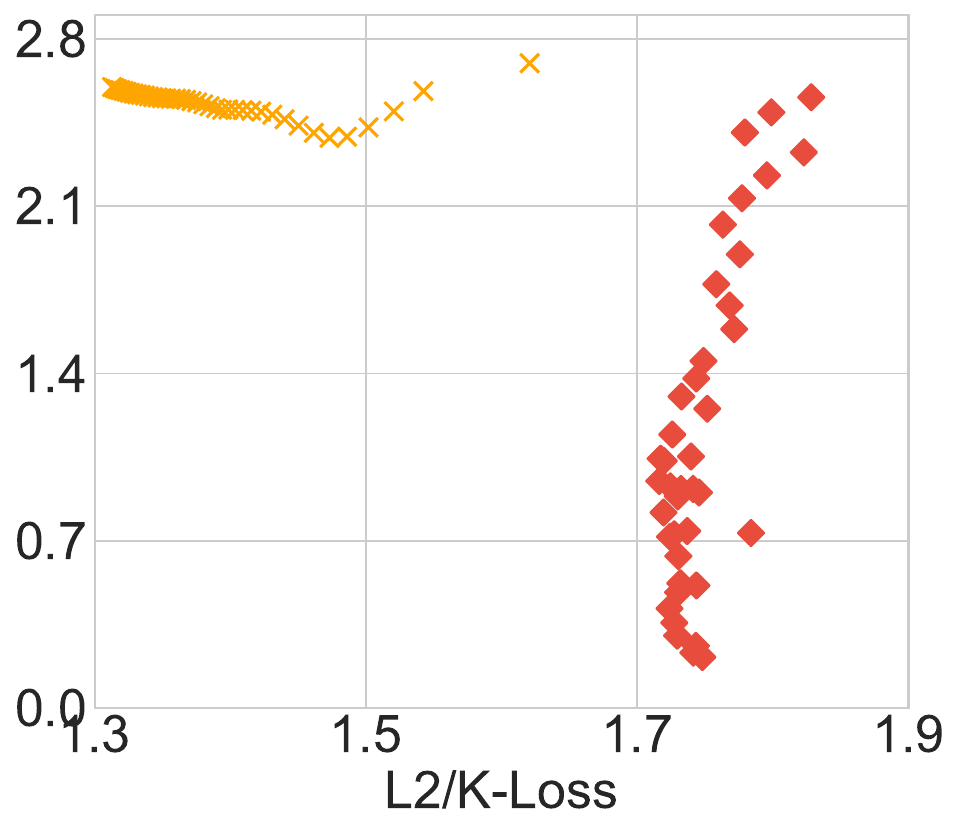} &
        \hspace{\gapline}
        \includegraphics[width=\tmoddd\linewidth]{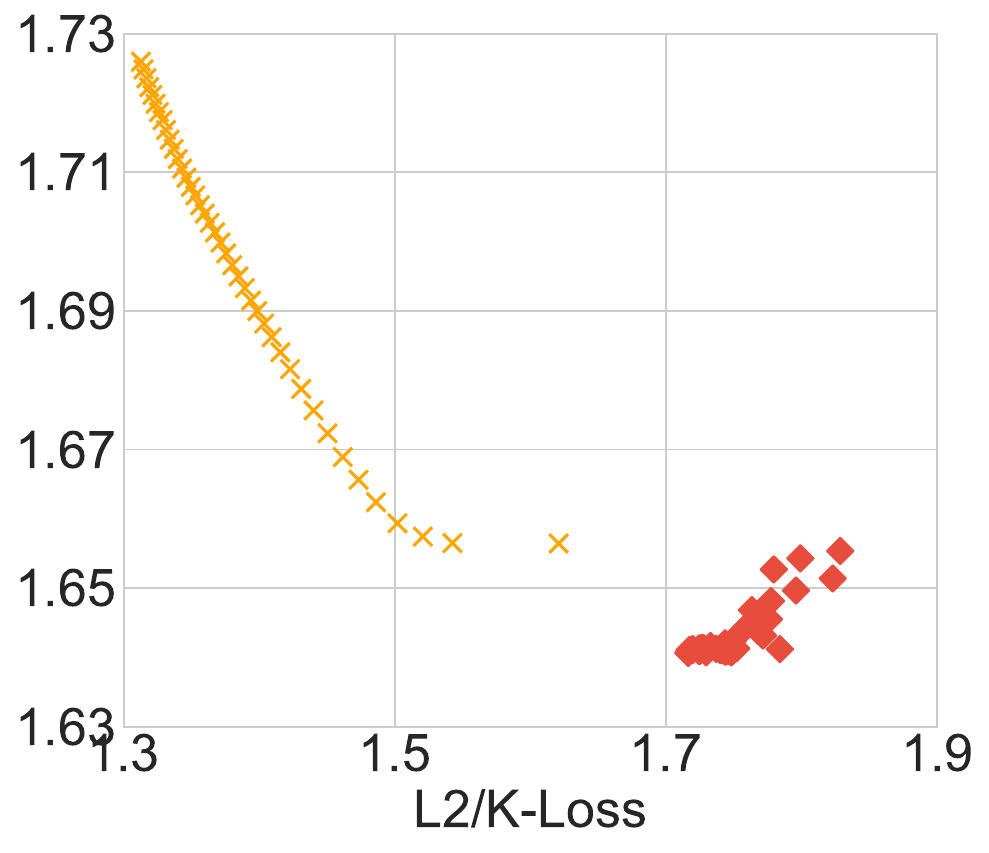} \hspace{-1.1em}\llap{\makebox[\wd1][r]{\raisebox{4.2em}{\includegraphics[height=4.0em]{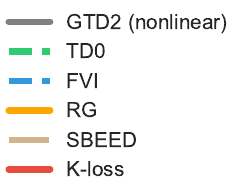}}}}
        \vspace{-.5em}
        \\
        \scriptsize{(a) MSE} &  \scriptsize{(b) Bellman Error} &  \scriptsize{(c) L2/K-Loss vs MSE} & 
        \scriptsize{(d)  L2/K-Loss vs Bellman Error} \\
    \end{tabular}
    \caption{ Results on Puddle World.}
    \vspace{-1em}
    \label{fig:pol_demo}
\end{figure}

\subsection{Policy Evaluation with Neural Networks}

While popular in recent RL literature, neural networks 
are known to be unstable for a long time.   Here, we revisit the classic divergence example of Puddle World~\citep{boyan95generalization}, and demonstrate the stability of our method.
Experimental details are found in Appendix~ \ref{sec:exp_eval}.

\figref{fig:pol_demo} summarizes the result using a neural network as value function for two metrics: $\|V-V^*\|_2^2$ and $\|\B V - V\|_2^2$, 
both evaluated on the training transitions. 
First, as shown in (a-b), our method works well while residual gradient converged to inferior solutions. 
In contrast, FVI and TD(0) exhibit unstable/oscilating behavior, and can even diverge, which is consistent with past findings~\citep{boyan95generalization}.
In addition, non-linear GTD2~\citep{maei09convergent} and SBEED~\citep{dai2017learning,dai18sbeed}, which do not find a better solution than our method in terms of MSE.

Second, \figref{fig:pol_demo} (c\&d) show the correlation between MSE, empirical Bellman error of the value function estimation and an algorithm's training objective respectively.  
Our kernel loss appears to be a good proxy for learning the value function, 
for both MSE and Bellman error.
In contrast, the L2 loss (used by residual gradient) does not correlate well, which also explains why residual gradient has been observed not to work well empirically.  %

\figref{fig:pol_bench} shows more results on value function learning on CartPole and Mountain Car,
which again demonstrate that our method performs better than other methods in general.

\newcommand{\ggapline}{-1.em}
\begin{figure}[!t]
    \centering
    \begin{tabular}{cccc}
        \hspace{\ggapline}
        \includegraphics[width=\tmoddd\linewidth]{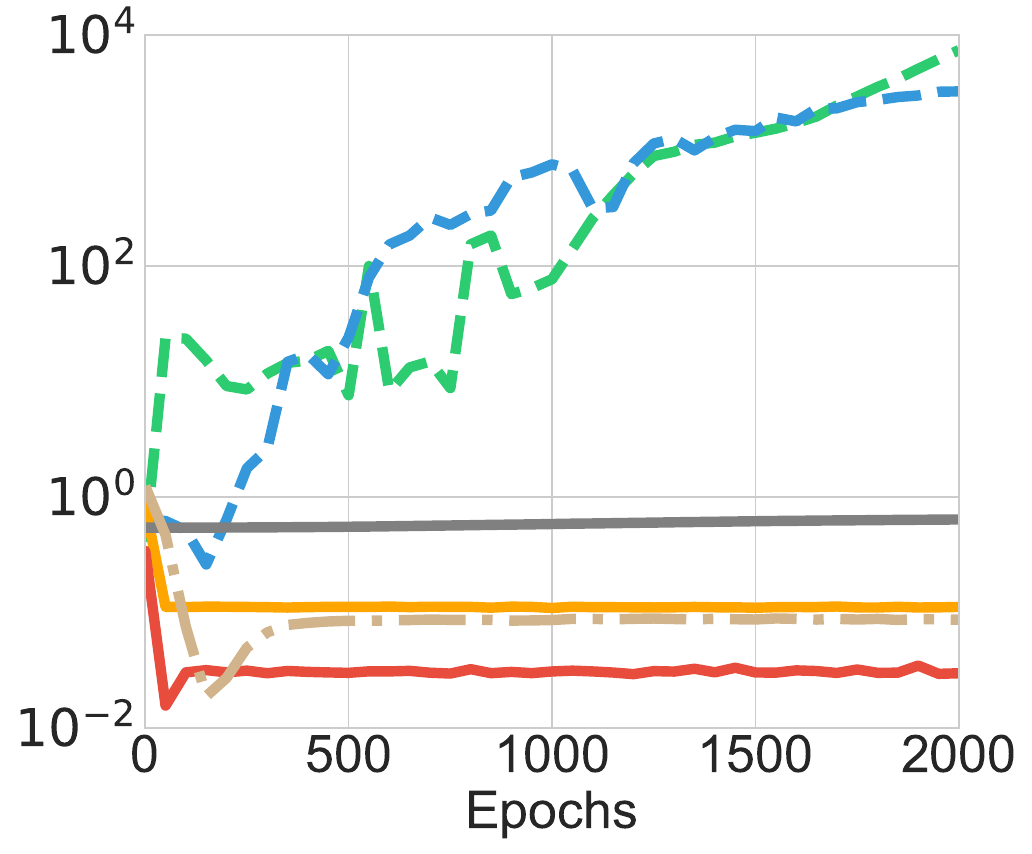}&
        \hspace{\ggapline}
        \includegraphics[width=\tmoddd\linewidth]{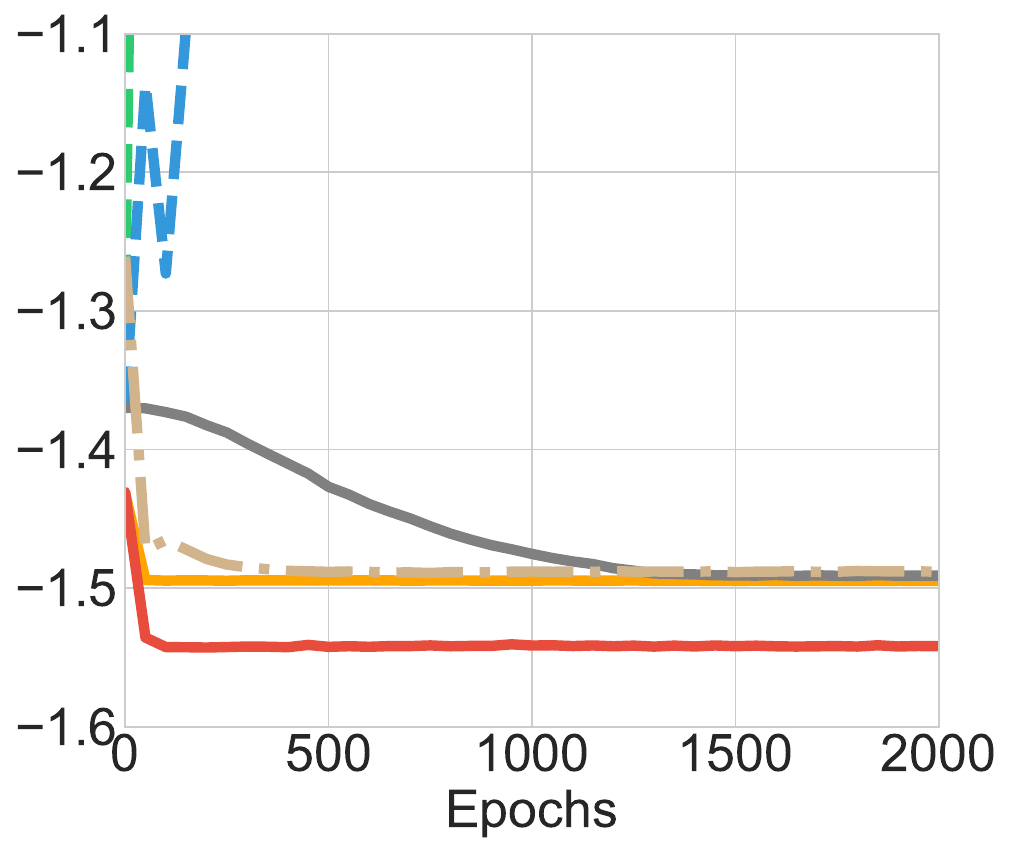}\hspace{-1.em}\llap{\makebox[\wd1][r]{\raisebox{4.8em}{\includegraphics[height=3.3em]{figures/puddle/legend.pdf}}}}& 
        \includegraphics[width=\tmoddd\linewidth]{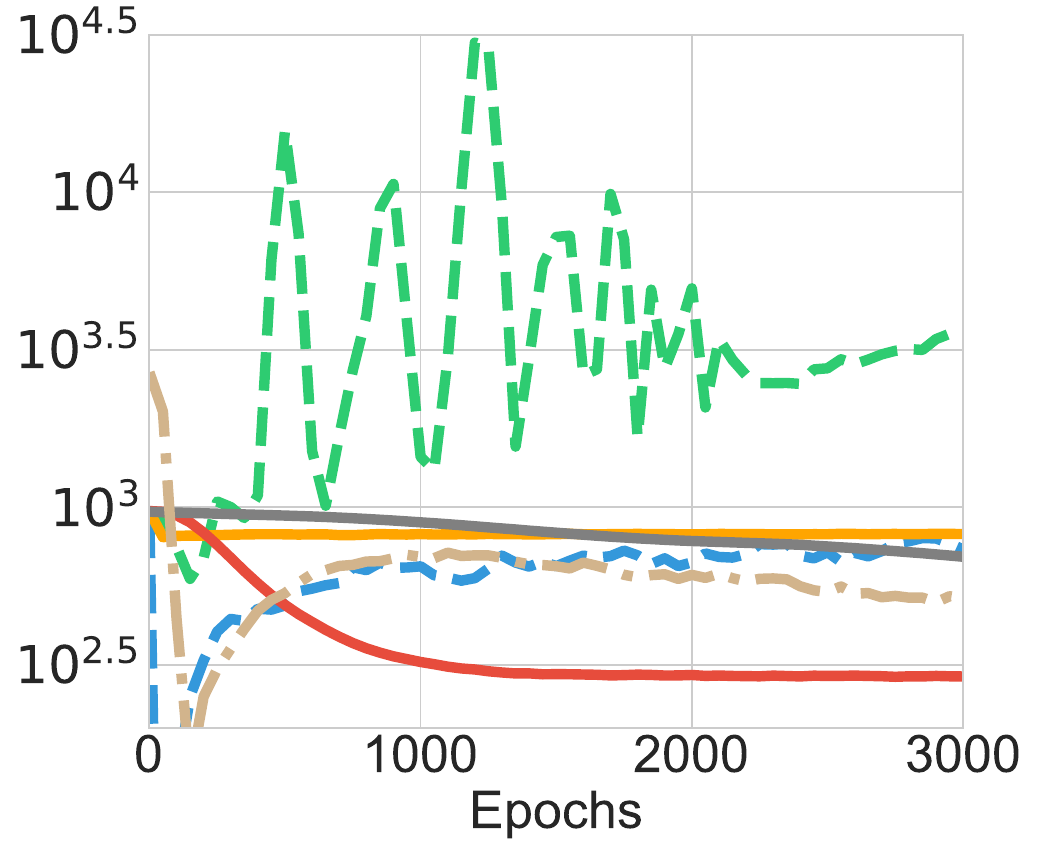}&
        \hspace{\ggapline}
        \includegraphics[width=\tmoddd\linewidth]{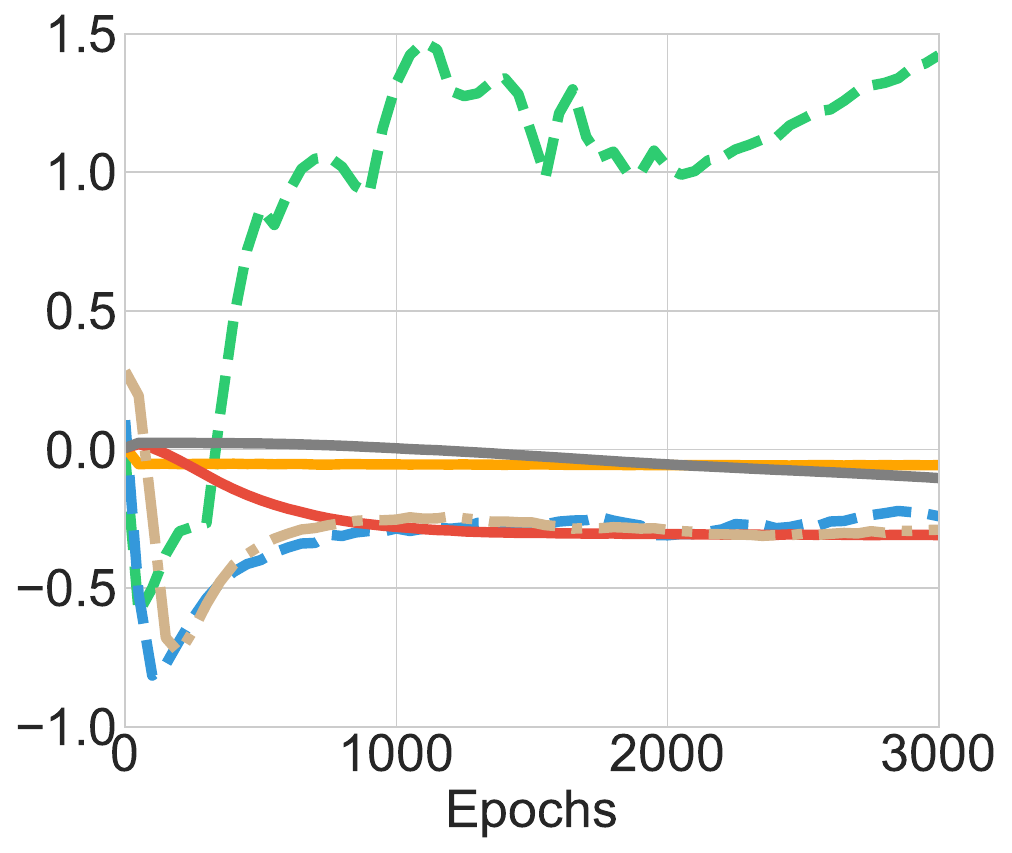} \\
        \vspace{-.2em}
        \scriptsize{(a) CartPole MSE} &  \scriptsize{(b) CartPole Bellman Error} &  \scriptsize{(c) Mountain Car MSE} & 
        \scriptsize{(d)  Mountain Car Bellman Error} \\
    \end{tabular}
    \caption{ Policy evaluation results on CartPole and Mountain Car.}
    \label{fig:pol_bench}
\end{figure}

\newcommand{\tmpleng}{.18}

\newcommand{\gapp}{-1.1em}
\begin{figure*}[!t]
    \centering
\hspace{-.04\textwidth}
    \begin{tabular}{cccc} 
        \raisebox{2em}{\rotatebox{90}{\scriptsize{Average Returns}}}
        \hspace{-.1em}
         \includegraphics[height=\tmpleng\linewidth]{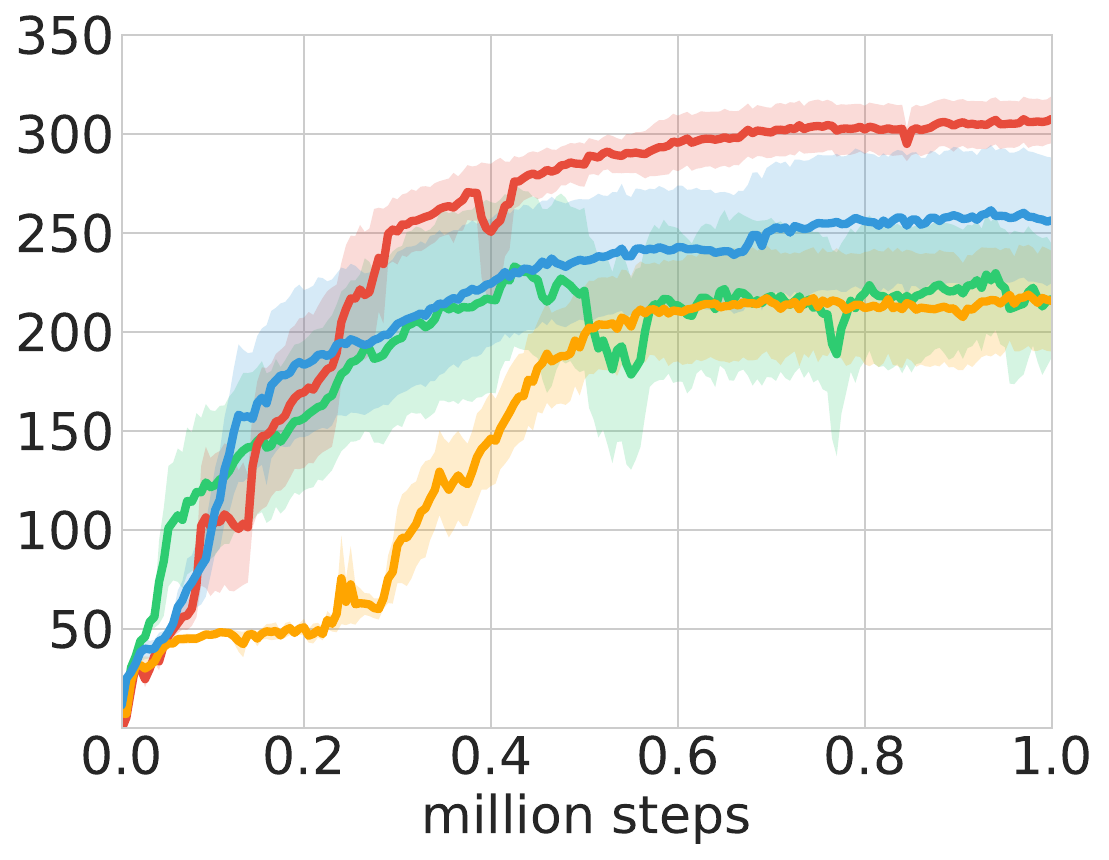}&  
         \hspace{\gapp}
          \includegraphics[height=\tmpleng\linewidth]{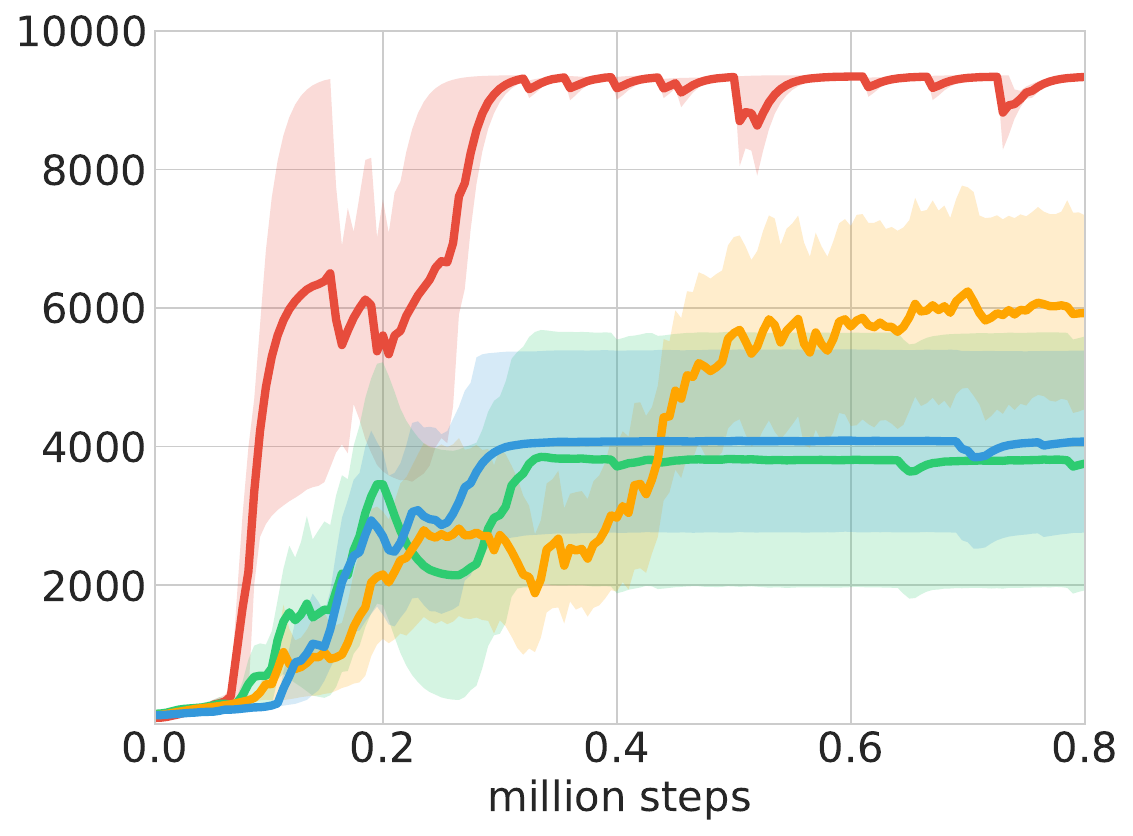}&  
          \hspace{\gapp}
          \includegraphics[height=\tmpleng\linewidth]{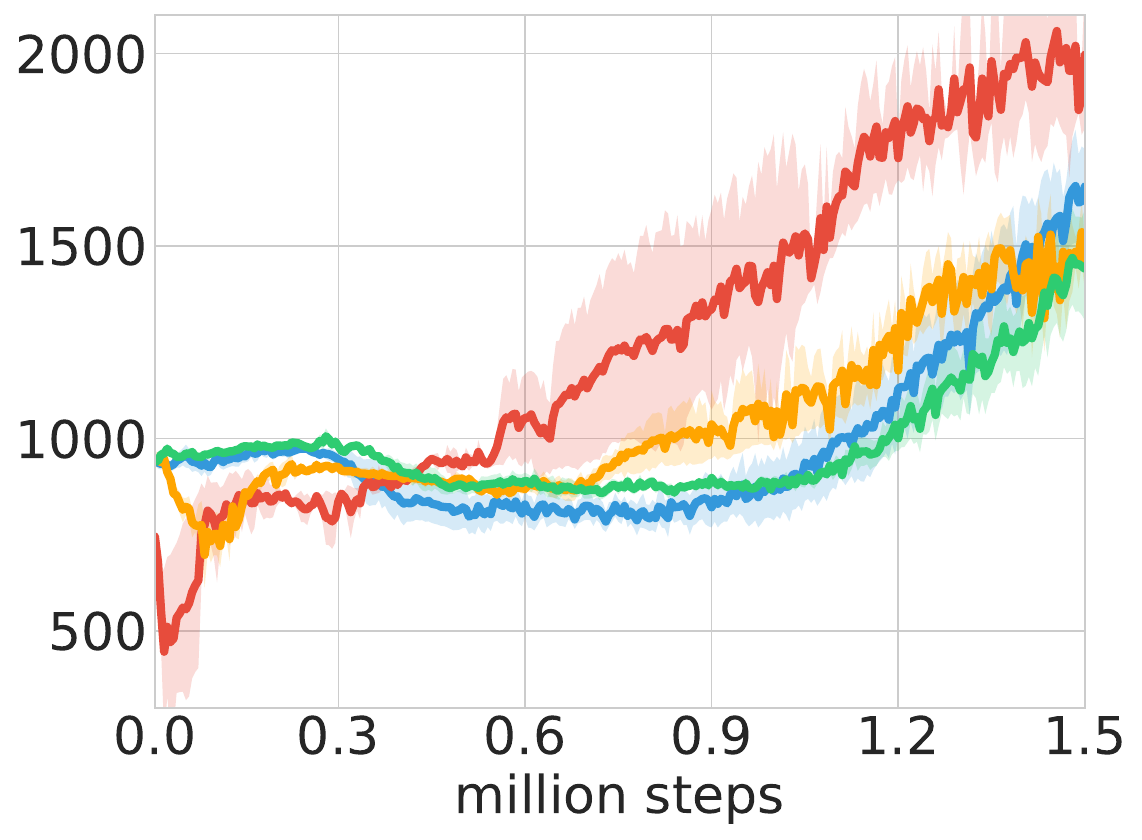} &
          \hspace{\gapp}
          \includegraphics[height=\tmpleng\linewidth]{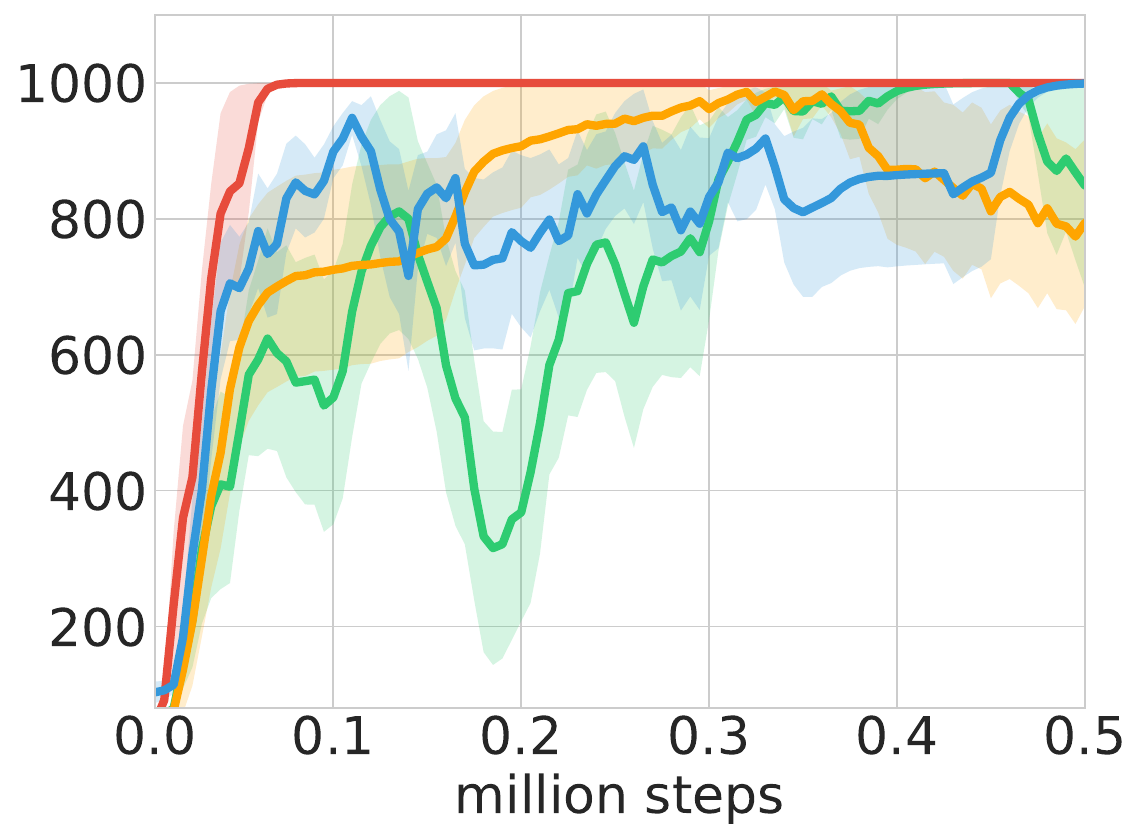}
          \hspace{-.85em}\llap{\makebox[\wd2][r]{\raisebox{1.1em}{\includegraphics[height=3.em]{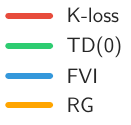}}}}
          \\
          \scriptsize{(a) Swimmer} & \scriptsize{(b) InvertedDoublePendulum} & \scriptsize{(c) Ant}  & \scriptsize{(d) InvertedPendulum}\\

    \end{tabular}
    \caption{ Results of various variants of Trust-PCL on Mujoco Benchmark.}
    \label{fig:mujoco_bench}
\end{figure*}

\subsection{Policy Optimization}
To demonstrate the use of our method in policy optimization, 
we combine it with Trust-PCL, 
and compare with variants of Trust-PCL combined with FVI, TD(0) and RG. 
To fairly evaluate the performance of all  these four methods, we use Trust-PCL  \citep{nachum2018trustpcl} framework and the public code for our experiments.  
We only modify the training of  $V_\theta(s)$ for each of the method and keep rest same as original release. Experimental details 
can be found in Appendix~\ref{sec:net_arc}.

We evaluate the performance of these four methods on Mujoco benchmark and report the best performance of these four methods in Figure \ref{fig:mujoco_bench} (averaged on five different random seeds). 
\KBO consistently outperforms all the other methods, learning better policy with fewer data.
Note that we only modify the update of value functions inside Trust-PCL, which can be implemented relatively easily. 
We expect that we can improve many other algorithms in similar ways
by improving the value function using our kernel loss. %

\section{Conclusion}
This paper studies the fundamental problem of solving Bellman equations with parametric value functions.  A novel kernel loss is proposed, which is easy to be estimated and optimized using sampled transitions.  Empirical results show that, compared to prior algorithms, our method is convergent, produces more accurate value functions, and can be easily adapted for policy optimization.

These promising results open the door to many interesting directions for future work.  An important question is finite-sample analysis, quantifying how fast the minimizer of the empirical kernel loss converges to the true minimizer of the population loss, when data is \emph{not} i.i.d.  Another is to extend the loss to the online setting, where data arrives in a stream and the learner cannot store all previous data.  Such an online version may provide computational benefits in certain applications.  Finally, it may be possible to quantify uncertainty in the value function estimate, and use this uncertainty information to guide efficient exploration.

\section*{Acknowledgment}
This work is supported in part by NSF CRII 1830161 and NSF CAREER 1846421. We would like to acknowledge Google Cloud and and Amazon Web Services (AWS) for their support.  We also thank an anonymous reviewer and Bo Dai for helpful suggestions on related work that improved the paper.
\clearpage

\bibliography{kbe}
\bibliographystyle{icml2019}

\clearpage
\appendix
\onecolumn
\begin{center}
\Large
\textbf{Appendix}
\end{center}

\section{Proofs for \secref{sec:theory}}

\subsection{Proof of Theorem~\ref{thm:validate}}

The assertion that $\Lk(V) \ge 0$ for all $V$ is immediate from definition.  For the second part, we have
\begin{eqnarray*}
    \Lk(V) = 0
&\Longleftrightarrow& \norm{\R_\pi V}_{\k,\mu} = 0  \\
&\Longleftrightarrow& \norm{\R_\pi V \cdot \mu}_\k = 0 \\
&\Longleftrightarrow& 
 \R_\pi V(s) \mu(s) \aseq 0, \quad \forall s\in\Sset 
\qquad\text{(since $\k$ is an is ISPD kernel)} \\
&\Longleftrightarrow& \R_\pi V(s) \aseq  0  \quad  \forall s\in\Sset \\
&\Longleftrightarrow& V \aseq V^\pi\,.
\end{eqnarray*}

\subsection{Proof of Theorem~\ref{thm:Vfunction}} 

\newcommand{\Ppi}{\mathcal{P}_{\pi}}
\newcommand{\I}{\mathcal{I}}

Define $\delta  = V - V^\pi$ to be the value function error.
Furthermore, let $\I$ be the identity operator ($\I V = V$), and
$$
\Ppi V(s) \defeq \E_{a \sim \pi(\cdot|s), s'\sim P(\cdot|s,a)}[ \gamma V(s') ~|~s]
$$
the state-transition part of Bellman operator without the local reward term $R(s,a)$.

Note that $\R_\pi V^\pi = \B_\pi  V^\pi  - V^\pi =  0$ by the Bellman equation, so
$$
\R_\pi V =
\R_\pi V - \R_\pi V^\pi 
= (\Ppi V - V) - (\Ppi V^\pi - V^\pi)
= (\Ppi-\I) (V - V^\pi) = (\Ppi-\I) \delta\,. 
$$
Therefore, 
\begin{align*}
    \Lk(V) 
    & = \E_{\mu} 
    [\mathcal \R_\pi V(s) \cdot \mathcal \R_\pi V(\bar s) \cdot \k(s,\bar s) ] \\
    & = \E_{\mu} 
    [\mathcal (\I-\Ppi)\delta(s) \cdot (\I-\Ppi)\delta(\bar s) \cdot \k(s,\bar s) ] \\
    &= \E_{(s,s'),(\bar s, \bar s') \sim d_{\pi,\mu}} [(\delta(s) - \gamma \delta(s')) \cdot (\delta(\sbar) - \gamma \delta(\sbar')) \cdot \k(s,\sbar) ], %
\end{align*}
where $\E_{d_{\pi,\mu}}[\cdot]$
denotes the expectation under 
the joint distribution 
$$
d_{\pi,\mu}(s, s') \defeq \mu(s) \sum_{a\in\Aset}\pi(a|s) P(s'|s,a) .
$$
Expanding the quadratic form above, we have 
\begin{align*}
    \lefteqn{\Lk(V)} \\
    &= \E_{d_{\pi,\mu}} [
    (\delta(s) \k(s,\bar s) \delta(\bar s) - \gamma \delta(s')\delta(\bar s) \k(s,\bar s) 
    -  \gamma \delta(\bar s')\delta(s) \k(s,\bar s) 
    + \gamma^2 \delta(s')  \delta(\bar s')  \k(s,\bar s) ] \\
    & = \E_{\mu}[\delta(s') \k^*(s',\bar s')\delta(\bar s')], 
\end{align*}
where $\k^*(s',\bar s')$ is as defined in the theorem statement:
\begin{align*} 
\k^*(s',\bar s') =
\E_{d^*_{\pi,\mu}}\left[ \k(s',\bar s') 
- \gamma (\k(s',\bar s)  
+ \k( s, \bar s'))
+ \gamma^2 \k(s, \bar s) ~|~ (s',\bar s')\right]
\end{align*}
with the expectation w.r.t. the following ``backward'' conditional probability 
$$
d^*_{\pi,\mu}(s~|~s') \defeq \frac{\sum_{a\in\Aset}\pi(a|s) P(s'|s,a) \mu(s)}{\mu(s')},
$$
which can be heuristically viewed as the distribution of state $s$ conditioning on observing its next state $s'$ when following $d_{\pi,\mu}(s,s')$. 

\subsection{Proof of Proposition~\ref{thm:eigen}}

Using the eigen-decomposition \eqref{equ:mercer},  we have 
\begin{align*}
    \Lk(V)
    & = \E_{\mu}[\R_\pi V(s) \k(s,\bar s) \R_\pi V(\bar s)] \\
    & =  \E_{\mu}[\R_\pi V(s)\sum_{i=1}^\infty \lambda_i e_i(s) e_i(\bar s)  \R_\pi V(\bar s)] \\ 
    & = \sum_{i=1}^\infty \lambda_i \left (\E_{\mu}[\R_\pi V(s) e_i(s)] \right)^2. 
\end{align*}
The decomposition of $L_2(V)$ follows directly from
Parseval's identity. 

\subsection{Proof of Proposition~\ref{thm:rkhs}}

The reproducing property of RKHS implies $f(s) = \langle f, ~\k(s,\cdot) \rangle_{\H_\k}$ for any $f \in \H_\k$. Therefore, 
\begin{align*}
 \E_{\mu}[\R_\pi V(s) f(s)] 
& = \E_{\mu} [\R_\pi V(s) \langle f, ~ \k(s, \cdot) \rangle_{\H_\k}] \\ 
& =  \langle f, ~ \E_{\mu} [\R_\pi V(s) \k(s, \cdot) ] \rangle_{\H_\k} \\  
& = \langle f, ~ f^* \rangle_{\H_\k}. 
\end{align*}
where we have defined $f^*(\cdot) \defeq \E_{\mu} [\R_\pi V(s) \k(s, \cdot) ]$. Maximizing $\langle f, ~~ f^* \rangle$ subject to $\norm{f}_{\H_\k} \defeq \sqrt{\langle f, f \rangle_{\H_\k}} \leq 1$ yields that $f = f^*/\norm{f^*}_{\H_\k}$.  Therefore,
$$
\max_{f \in \H_\k:~\norm{f}_{\H_\k}\le1}
 \left (\E_{s} \left [ \R_\pi V(s) f(s) \right ] \right )^2 = (\langle \frac{f^*}{\norm{f^*}_{\H_\k}}, ~ f^* \rangle_{\H_\k})^2 = \norm{f^*}_{\H_\k}^2.
$$
Further, we can show that 
\begin{eqnarray*}
\norm{f^*}_{\H_\k}^2 
&=& \langle f^*, ~ f^* \rangle_{\H_\k}  \\
&=& \langle \E_{\mu} [\R_\pi V(s) \k(s, \cdot) ], ~~ 
\E_{\mu} [\R_\pi V(\bar s) \k(\bar s, \cdot) ]  \rangle_{\H_\k} \\
&=& \E_{\mu} [\R_\pi V(s) \k(s,\bar s) \R_\pi V(\bar s) ],
\end{eqnarray*}
where the last step follows from the reproducing property, $\k(s,\bar s) = \langle \k(s,\cdot), \k(\bar s,\cdot) \rangle_{\H_\k}$.  This completes the proof, by definition of $\Lk(V)$.

\subsection{Proof of \corref{cor:same_as_td}}

Under the conditions of the corollary, the kernel loss becomes the Norm of the Expected TD Update (NEU), whose minimizer coincides with the TD solution~\citep{dann14policy}.  For completeness, we provide a self-contained proof.

Since we are estimating the value function of a fixed policy, we ignore the actions, and the set of transitions is $\Dset = \{(s_i,r_i,s_i')\}_{1 \le i \le n}$.  Define the following vector/matrices:
\begin{eqnarray*}
r &=& \left[ r_1; ~ r_2,\cdots;~ r_n\right] \in \Rset^{n\times 1}\,, \\
X &=& \left [\phi(s_1); ~\phi(s_2); ~\dots;~ \phi(s_n) \right ]  \in \Rset^{n\times d}\,, \\
X' &=& \left [ \phi(s_1'); ~\phi(s_2');~ \dots;~  \phi(s_n')\right ]  \in \Rset^{n\times d}\,,
\end{eqnarray*}
and $Z = X - \gamma X'$, where $d$ is the feature dimension.
Then, the TD solution is given by
\[
\thetaTD = (X^\mt Z)^\mi X^\mt r\,.
\]
Note that the above includes both the on-policy case as well as the off-policy case as in many previous algorithms with linear value function approximation~\citep{dann14policy}, where the difference is in whether $s_i$ is sampled from the state occupation distribution of the target policy or not.

Define $\delta\in\Rset^{n \times 1}$ to be the TD error vector; that is, $\delta = r - Z \theta$, where
$
\delta_i = r_i + \gamma V(s_i') - V(s_i) = r_i + \theta^\mt (\gamma \phi(s_i') - \phi(s_i))\,.
$
With a linear kernel, our objective function becomes:
\[
\ell(\theta) = \frac{1}{n^2} \sum_{i,j} \delta_i \k(s_i,s_j) \delta_j = \frac{1}{n^2} \delta^\mt X X^\mt \delta = \frac{1}{n^2} (r-Z\theta)^\mt X X^\mt (r-Z\theta) \,.
\]
Its gradient is given by
\begin{align*}
\nabla \ell &= \frac{2}{n^2}(Z^\mt X X^\mt Z \theta - Z^\mt X X^\mt r)\,.
\end{align*}
Letting $\nabla\ell=0$ gives the solution obtained by minimizing our kernel loss:\footnote{For simplicity, we assume all involved matrices of size $d \times d$ are non-singular, as is typical in analyzing TD algorithms.  Without this assumption, we may either add $L_2$-regularization to $XX^\mt$~\citep{farahmand08regularized}, for which the same equivalence between TD and ours can be proved, or show that the solutions lie in an affine space in $\Rset^d$ but the corresponding value functions are identical.}
\begin{align*}
\hat{\theta}_{\operatorname{KBE}} &= (Z^\mt X X^\mt Z)^\mi Z^\mt X X^\mt r\,.
\end{align*}
Therefore,
\begin{align*}
\hat{\theta}_{\operatorname{KBE}} - \hat{\theta}_{\operatorname{TD}} 
&= \left((Z^\mt X X^\mt Z)^\mi Z^\mt X - (X^\mt Z)^\mi\right) X^\mt r \\
&= \left((Z^\mt X X^\mt Z)^\mi Z^\mt X (X^\mt Z) - I\right) (X^\mt Z)^\mi X^\mt r \\
&= \left(I - I\right) (X^\mt Z)^\mi X^\mt r = 0\,.
\end{align*}

\section{Experiment Details}

\newcommand{\smallergap}{\vspace{-2mm}}

\subsection{Kernel Loss Estimation with Batch Samples}\label{sec:batch_lv}
Given a set of empirical data $\D = \{(s_i,a_i,r_i,s_i')\}_{1\le i \le n}$, where $n$ is large such that we need to use a subset samples $\B = \{(S_i,A_i,R_i,S_i')\}_{1\le i \le m}$ drawn from $\D$ to estimate the empirical kernel loss.
One way to estimate $\Lk$ using the subset $\B$ is \emph{U-statistics},
\begin{align*}
    \hat{L}_{KU}(V_\theta) 
\defeq \frac{1}{m(m-1)}
\sum_{1 \le i \ne j \le m}
\k(S_i, S_j) \cdot \hat\R_\pi V_\theta(S_i) \cdot \hat\R_\pi V_\theta (S_j)\,.
\end{align*}

Similarly, we can use the \emph{V-statistics} to estimate $L_{K}$ given the subset $\B$:
\begin{align*}
\hat{L}_{KV}(V_\theta) \defeq \frac{1}{mn} & \Bigg(\bigg(\sum_{1 \le i \le m}\k(S_i, S_i) \cdot \hat\R_\pi V_\theta(S_i) \cdot \hat\R_\pi V_\theta (S_i)\bigg)  \\
& + \frac{n-1}{m-1} \bigg(\sum_{1\le i\ne j \le m}\k(S_i, S_j) \cdot \hat\R_\pi V_\theta(S_i) \cdot \hat\R_\pi V_\theta (S_j)\bigg)\Bigg)\,.
\end{align*}

In our experiments, we observe that V-statistics works slightly better than U-statistics, and we use an mixed combination of these two to achieve better performance: $\alpha\hat{L}_{KV}(V_\theta) +(1 - \alpha) \hat{L}_{KU}(V_\theta)$, where $\alpha$ is a coefficient which can be tunned.

\subsection{Policy Evaluation}
\label{sec:exp_eval}
We compare our method with representative policy evaluation methods including TD(0), FVI, RG, nonlinear GTD2~\citep{maei09convergent} and SBEED~\citep{dai2017learning,dai18sbeed} on three different stochastic environments: Puddle World, CartPole and Mountain Car. 
Followings are the detail of the policy evaluation experiments.
\smallergap

\paragraph{Network Structure}
We parameterize the value function $V_{\theta}(s)$ using a fully connected neural network with one hidden layer of $80$ units, 
using \texttt{relu} as activation function.
For test function $f(s)$ in SBEED, 
we use a small neural network with $10$ hidden units and \texttt{relu} as activation function.
\smallergap

\paragraph{Data Collection}
For each environment, we randomly collect $5000$ independent transition tuples with states uniformly drawn from state space using a policy $\pi$ learned by policy optimization,
for which we want to learn the value function $V^{\pi}(s)$.
\smallergap

\paragraph{Estimating the true value function $V^\pi(s)$}
To evaluate and compare all methods, we approximate the true value function by finely discretizing the state space and then applying tabular value iteration on the discretized MDP.
Specifically, we discretize the state space into $25 \times 25$ grid for Puddle World,
 $20 \times 25$ discrete states for CartPole, and 
$30 \times 25$ discrete states for Mountain Car.
\smallergap

\paragraph{Training Details}
For each environment and each policy evaluation method,
we train the value function $V_{\theta}(s)$ on the collected $5000$ transition tuples for $2000$ epochs ($3000$ for Mountain Car), with a batch size $n=150$ in each epoch using Adam optimizer.
We search the learning rate in $\{0.003, 0.001, 0.0003\}$ for all methods and report the best result averaging over $10$ trials using different random seeds.
For our method, we use a Gaussian RBF kernel $\k(s_i, s_j) = \exp{(-\norm{s_i - s_j}^2_2/h^2)}$ and take the bandwidth to be $h=0.5$.
For FVI, we update the target network at the end of each epoch training.
For SBEED, we perform $10$ times gradient ascent updates on the test function $f(s)$ and $1$ gradient descent update on $V_{\theta}(s)$ at each iteration.
We fix the discount factor to $\gamma = 0.98$ for all environments and policy evaluation methods.

\subsection{Policy Optimization}
\label{sec:exp_detail}
In this section we describe in detail the experimental setup for policy optimization regarding implementation and hyper-parameter search.  The code of Trust-PCL is available at github.\footnote{\url{https://github.com/tensorflow/models/tree/master/research/pcl_rl}}
Algorithm~\ref{alg:kloss_pcl} describes details in pseudocode, where the the main change compared to Trust-PCL is highlighted.  
Note that as in previous work, we use the $d$-step version of Bellman operator, an immediate extension to the $d=1$ case described in the main text.

\subsubsection{Network Architectures}
\label{sec:net_arc}
We use fully-connected feed-forward neural network to represent both policy and value network.
The policy $\pi_{\theta}$ is represented by a neural network with $64 \times 64$ hidden layers with \texttt{tanh} activations. 
At each time step $t$, the next action $a_t$ is sampled randomly from a Gaussian distribution $\mathcal{N}(\mu_\theta(s_t), \sigma_\theta)$.
The value network $V_\theta(s)$ is represented by a neural network with $64 \times 64$ hidden layers with \texttt{tanh} activations. At each time step $t$, the network is given the observation $s_t$ and it produces a single scalar output value. All methods share the same policy and value network architectures.

\begin{algorithm}[t]
\small
\caption{K-Loss for PCL}
\label{alg:kloss_pcl}
\begin{algorithmic}
\STATE \textbf{Input:} rollout step $d$, batch size $B$, coefficient $\lambda,~ \tau, \alpha$.
\STATE Initialize $V_\theta(s)$, $\pi_{\phi}(a|s)$, and empty replay buffer $RB(\beta)$. Set $\tilde{\phi} = \phi$.
\REPEAT
\vspace{1mm}
\STATE \textit{// Collecting Samples}
\STATE Sample $P$ steps $s_{t:t+P} \sim \pi_{\phi}$ on ENV.
\STATE Insert  $s_{t: t+P}$ to $RB(\beta)$.
\vspace{2mm}
\STATE \textit{// Train}
\STATE Sample batch $\{s_{t:t+d}^{(k)}, a_{t:t+d}^{(k)}, r_{t:t+d}^{(k)}\}_{k=1}^B$ from $RB(\beta)$ to contain a total of $Q$ transitions ($B \approx Q / d$).
\STATE \textcolor{blue}{$\Delta \theta= \alpha\nabla_{\theta}\hat{L}_{KV}(V_\theta) + (1 - \alpha)\nabla_{\theta}\hat{L}_{KU}(V_\theta)$,}
\STATE $\Delta \phi = -\frac{1}{B}\sum_{1 \leq i \leq B}[\hat{\mathcal{R}}_i\sum_{t=0}^{d-1}\nabla_{\phi}\log\pi_{\phi}(a_{t+i}|s_{t+i})]$, where
$$
\hat{\mathcal{R}}_i = -V_{\theta}(s_i) + \gamma^{d}V_{\theta}(s_{i+d}) + \sum_{t=0}^{d-1}\gamma^{t}(r_{i+t} - (\lambda + \tau)\log \pi_{\phi}(a_{t+i} | s_{t+1}) + \tau \log \pi_{\tilde{\phi}}(a_{t+i} | s_{t+1})).
$$
\STATE Update $\theta$ and $\phi$ using ADAM with $\Delta\theta, \Delta\phi$.
\vspace{2mm}
\STATE \textit{// Update auxiliary variables}
\STATE Update $\tilde{\phi} = \alpha \tilde{\phi}  + (1 - \alpha )\phi$.
\vspace{1mm}
\UNTIL{Convergence}
\end{algorithmic}
\end{algorithm}

\begin{figure}[t]
    \centering
    \setbox1=\hbox{\includegraphics[height=10.2em]{figures/benchmark2/legend_bench.pdf}}
    \begin{tabular}{cc}
         \includegraphics[width=.32\linewidth]{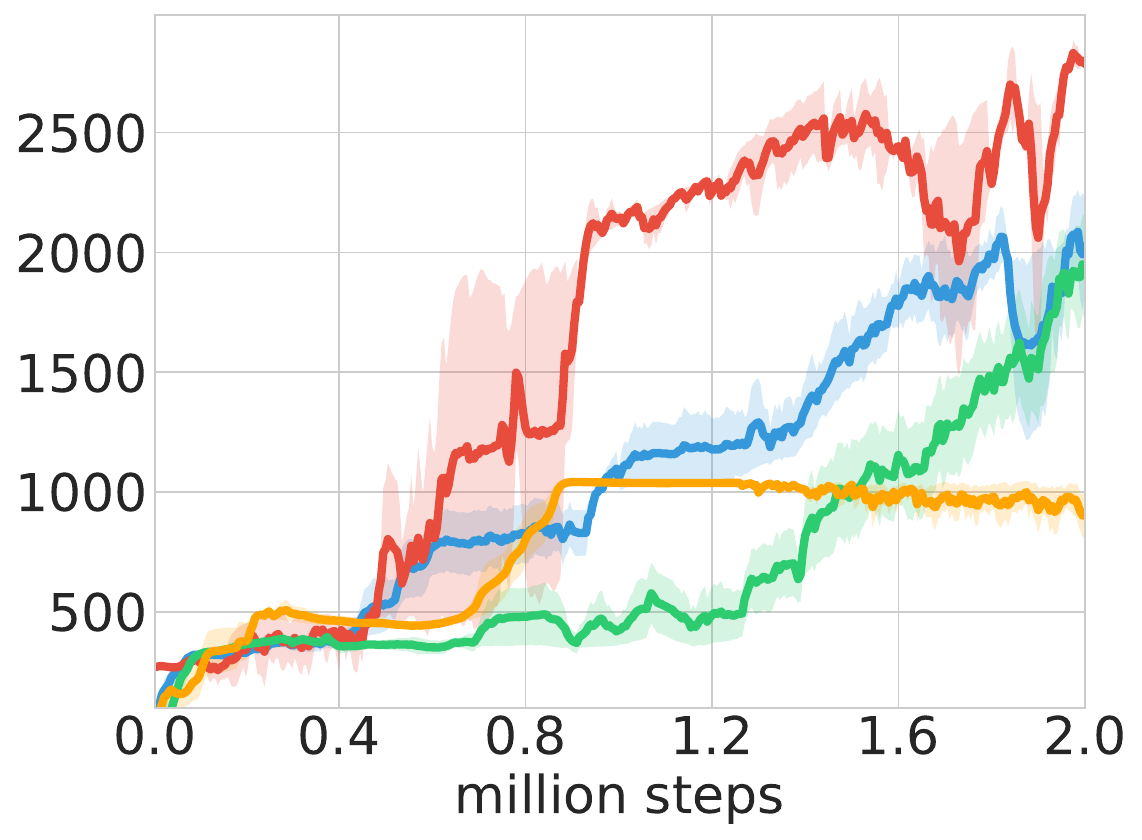}& 
         \hspace{2.5em}
          \includegraphics[width=.32\linewidth]{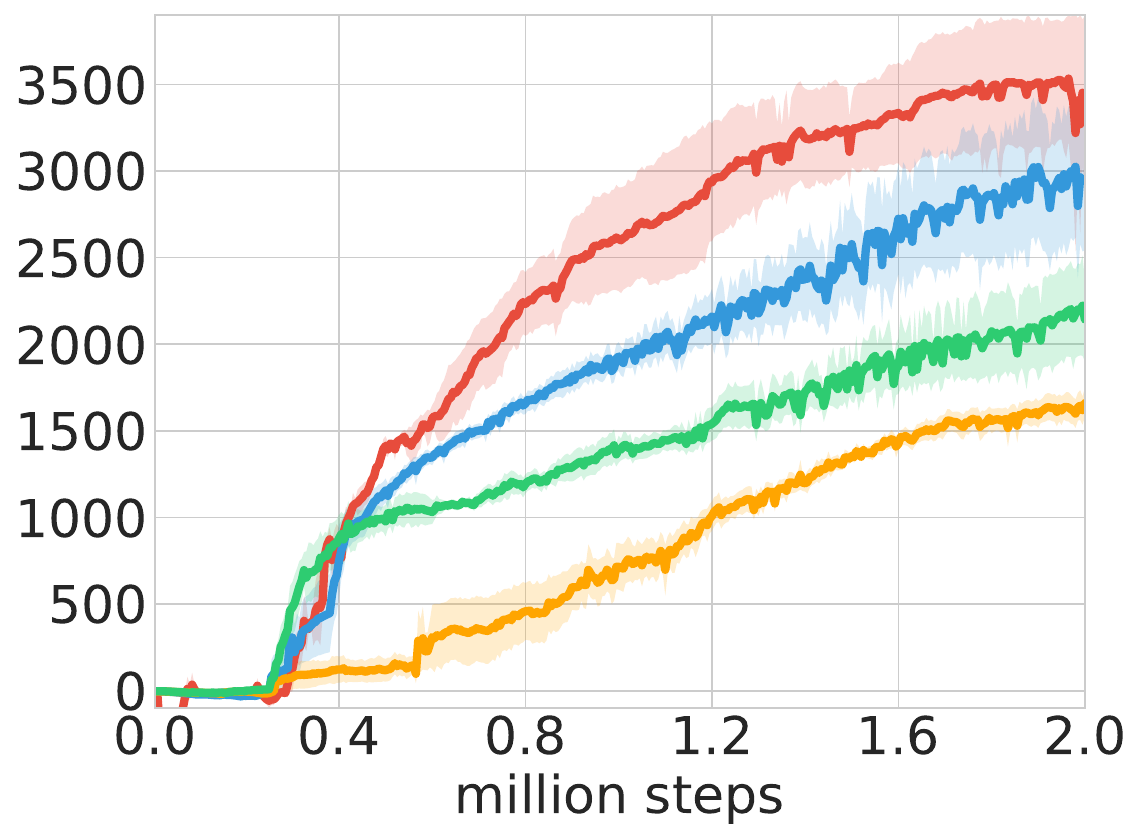}\llap{\makebox[\wd1][l]{\raisebox{5.8em}{\includegraphics[height=3.3em]{figures/benchmark2/legend_bench.pdf}}}}\\
          (a) \footnotesize Walker2d & (b) \footnotesize HalfCheetah \\ 
    \end{tabular}
    \caption{More results of various variants of Trust PCL on Mujoco Benchmark (on top of Figure~\ref{fig:mujoco_bench}).}
    \label{fig:more_bench}
\end{figure}

\subsubsection{Training Details}
We average over the best 5 of 6 randomly seeded training runs and evaluate each method using the mean $\mu_\theta (s)$ of the diagonal Gaussian policy $\pi_{\theta}$. 
Since Trust-PCL is off-policy, we collect experience 
and train on batches of experience sampled from the replay buffer. 
At each training iteration, we will first sample $T=10$ timestep samples and add them to the replay buffer, then both the policy and value parameters are updated in a single gradient step using the Adam optimizer with a proper learning rate searched, using a minibatch randomly sampled from replay buffer.
For Trust-PCL using FVI updating $V_\theta(s)$, which requires a target network to estimate the final state for each path, we use an exponentially moving average, with a smoothing constant $\tau=0.99$, to update the target value network weights as common in the prior work \citep{mnih15human}.
For Trust-PCL using TD(0), we will directly use current value network $V_\theta(s)$ to estimate the final states except we do not perform gradient update for the final states.
For Trsut-PCL using RG and K-loss, which has an objective loss, we will directly perform gradient descent to optimize both policy and value parameters.

\subsubsection{Hyperparameter Search}
We follow the same hyperparameter search procedure in \citet{nachum2018trustpcl} for FVI, TD(0) and RG based Trust-PCL.\footnote{See README in \url{https://github.com/tensorflow/models/tree/master/research/pcl_rl}} 
We search the maximum divergence $\epsilon $ between $\pi_{\theta}$ and $\pi_{\hat \theta}$ among $\in \{0.001, 0.0005, 0.002\}$, and parameter learning rate in $\{0.001, 0.0003, 0.0001\}$, and the rollout length $d \in \{1, 5, 10\}$. We also searched with the entropy coefficient $\lambda$, either keeping it at a constant $0$ (thus, no exploration) or decaying it from $0.1$ to $0.0$ by a smoothed exponential rate of $0.1$ every $2500$ training iterations. For each hyper-parameter setting, we average best $5$ of $6$ seeds and report the best performance for these methods.

For our proposed K-loss, we also search the maximum divergence $\epsilon$ but keep the learning rate as $0.001$. Additionally, for K-loss we use a Gaussian RBF kernel $\k([s_i, a_i], [s_j, a_j]) = \exp{(-(\norm{s_i - s_j}^2_2 + \norm{a_i - a_j}^2_2)/h)}$, and take the bandwidth to be $h =(\alpha \times \rm{med})^2$, where we search $\alpha \in \{0.1, 0.01, (1 / \sqrt{\log B})\}$, and $B=64$ is the gradient update batch size.
We fix the discount to $\gamma=0.995$ for all environments and batch size $B = 64$ for each training iteration.

\clearpage

\end{document}